%% file: main.tex
\documentclass[10pt,twocolumn,letterpaper]{article}

\usepackage{cvpr}              %

\input{preamble}

\usepackage{multirow}
\usepackage{booktabs} %
\usepackage[table]{xcolor}
\usepackage{mathtools}
\usepackage{graphicx}   %
\usepackage{makecell}   %
\usepackage{booktabs}   %
\usepackage[percent]{overpic}  %
\definecolor{cvprblue}{rgb}{0.21,0.49,0.74}
\usepackage[pagebackref,breaklinks,colorlinks,allcolors=cvprblue]{hyperref}

\usepackage{thm-restate}

\def\paperID{41401} %
\def\confName{CVPR}
\def\confYear{2026}

\title{Designing to Forget: Deep Semi-parametric Models for Unlearning}

\author{Amber Yijia Zheng$^\star$ \hspace{.2cm} \and Yu-Shan Tai$^\star$ \hspace{.2cm} \and Raymond A. Yeh \hspace{.2cm} \and
Department of Computer Science, Purdue University\\
}

\begin{document}

\twocolumn[{%
\renewcommand\twocolumn[1][]{#1}%
\maketitle
\vspace{-0.65cm}
\input{figs/teaser}

\vspace{.4cm}
}]
\let\thefootnote\relax\footnotetext{$^\star$ indicates equal contribution}
\input{src/abs}
\input{src/intro}
\input{src/rel}
\input{src/prelim}

\input{src/app}

\input{src/exp}

\input{src/conc}

\clearpage

{
    \bibliography{ref}
    \small
    \bibliographystyle{ieeenat_fullname}
}

\clearpage

\input{src/supp}

\end{document}

%% file: preamble.tex
\usepackage{amssymb}%
\usepackage{pifont}%

\usepackage{symbols}
\input{math_commands}

\newcommand*{\ShowNotes}{} %
\input{macros}

\usepackage{amsmath,amssymb,amsthm}
\DeclareMathOperator{\Acc}{Acc}

%% file: math_commands.tex
\usepackage{amsmath,amsfonts,bm}

\def\1{\bm{1}}

\def\sp{space}

\def\vb{{\bm{b}}}

\def\vs{{\bm{s}}}

\def\vw{{\bm{w}}}
\def\vx{{\bm{x}}}
\def\vy{{\bm{y}}}
\def\vz{{\bm{z}}}

\def\mW{{\bm{W}}}

\DeclareMathAlphabet{\mathsfit}{\encodingdefault}{\sfdefault}{m}{sl}
\SetMathAlphabet{\mathsfit}{bold}{\encodingdefault}{\sfdefault}{bx}{n}
\newcommand{\tens}[1]{\bm{\mathsfit{#1}}}

\def\tX{{\tens{X}}}

\def\gD{{\mathcal{D}}}

\def\gL{{\mathcal{L}}}

\def\gN{{\mathcal{N}}}

\def\gQ{{\mathcal{Q}}}

\def\gS{{\mathcal{S}}}
\def\gT{{\mathcal{T}}}
\def\gU{{\mathcal{U}}}
\def\gV{{\mathcal{V}}}

\def\sR{{\mathbb{R}}}
\def\sS{{\mathbb{S}}}

\newcommand{\KL}{D_{\mathrm{KL}}}

\DeclareMathOperator*{\argmax}{arg\,max}
\DeclareMathOperator*{\argmin}{arg\,min}

%% file: macros.tex
\usepackage{soul}
\usepackage{multirow}

\definecolor{lightred}{rgb}{0.9,0.4,0.4}
\definecolor{darkred}{rgb}{0.7,0.1,0.1}
\definecolor{darkgreen}{rgb}{0.1,0.7,0.1}
\definecolor{cyan}{rgb}{0.7,0.0,0.7}
\definecolor{dblue}{rgb}{0.2,0.2,0.8}
\definecolor{maroon}{rgb}{0.76,.13,.28}
\definecolor{burntorange}{rgb}{0.81,.33,0}
\definecolor{tealblue}{rgb}{0.212,0.459, 0.533}
\definecolor{myyellow}{rgb}{0.8627451 , 0.75294118, 0.20784314]}
\definecolor{mypink}{rgb}{0.93359375, 0.62109375, 0.83984375}
\definecolor{pp}{rgb}{0.43921569, 0.18823529, 0.62745098}
\definecolor{rr}{rgb}{0.5254902 , 0.00784314, 0.12941176}
\definecolor{bb}{rgb}{0.09019608, 0.23529412, 0.37647059}
\definecolor{yy}{rgb}{0.49803922, 0.3372549 , 0.0}
\definecolor{gg}{rgb}{0.02352941, 0.3372549 , 0.17647059}
\definecolor{tblue}{HTML}{01847F}

\definecolor{retrain}{HTML}{D3D3D3}

\ifdefined\ShowNotes
  \newcommand{\colornote}[3]{{\color{#1}\bf{#2: #3}\normalfont}}
\else
  \newcommand{\colornote}[3]{}
\fi

\newcommand{\myparagraph}[1]{\vspace*{1.5pt}{\bf\noindent #1}}

\newcommand{\eat}[1]{} %

\newlength\savewidth

\makeatletter
\renewcommand{\paragraph}{%
  \@startsection{paragraph}{4}%
  {\z@}{0.3ex \@plus 1ex \@minus .1ex}{-1em}%
  {\normalfont\normalsize\bfseries}%
}
\makeatother

\newif\ifproofread

%% file: figs/teaser.tex
\small
\centering
\setlength{\tabcolsep}{1pt}
\begin{tabular}{ccc}
\includegraphics[width=0.31\linewidth]{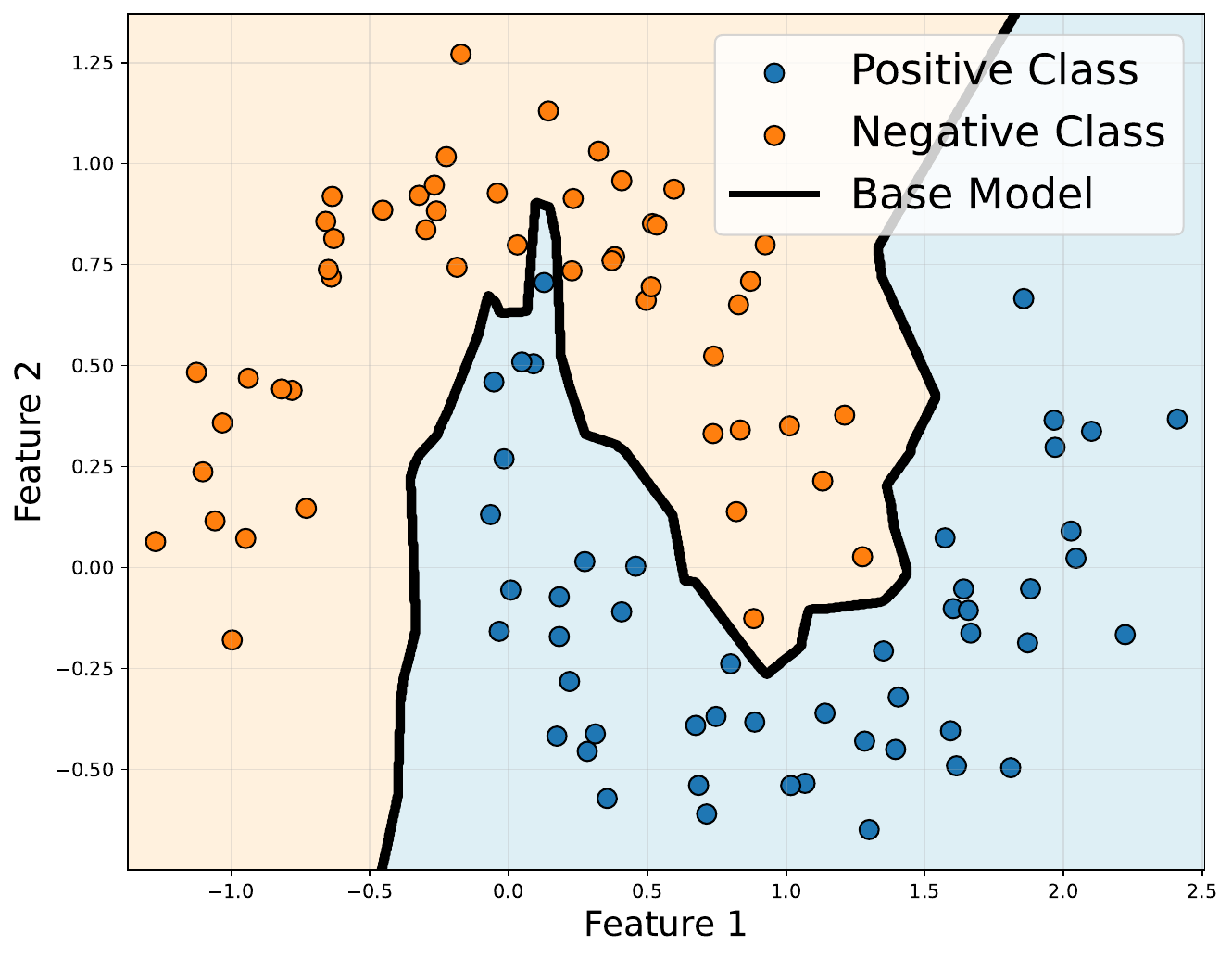} & 
\includegraphics[width=0.31\linewidth]{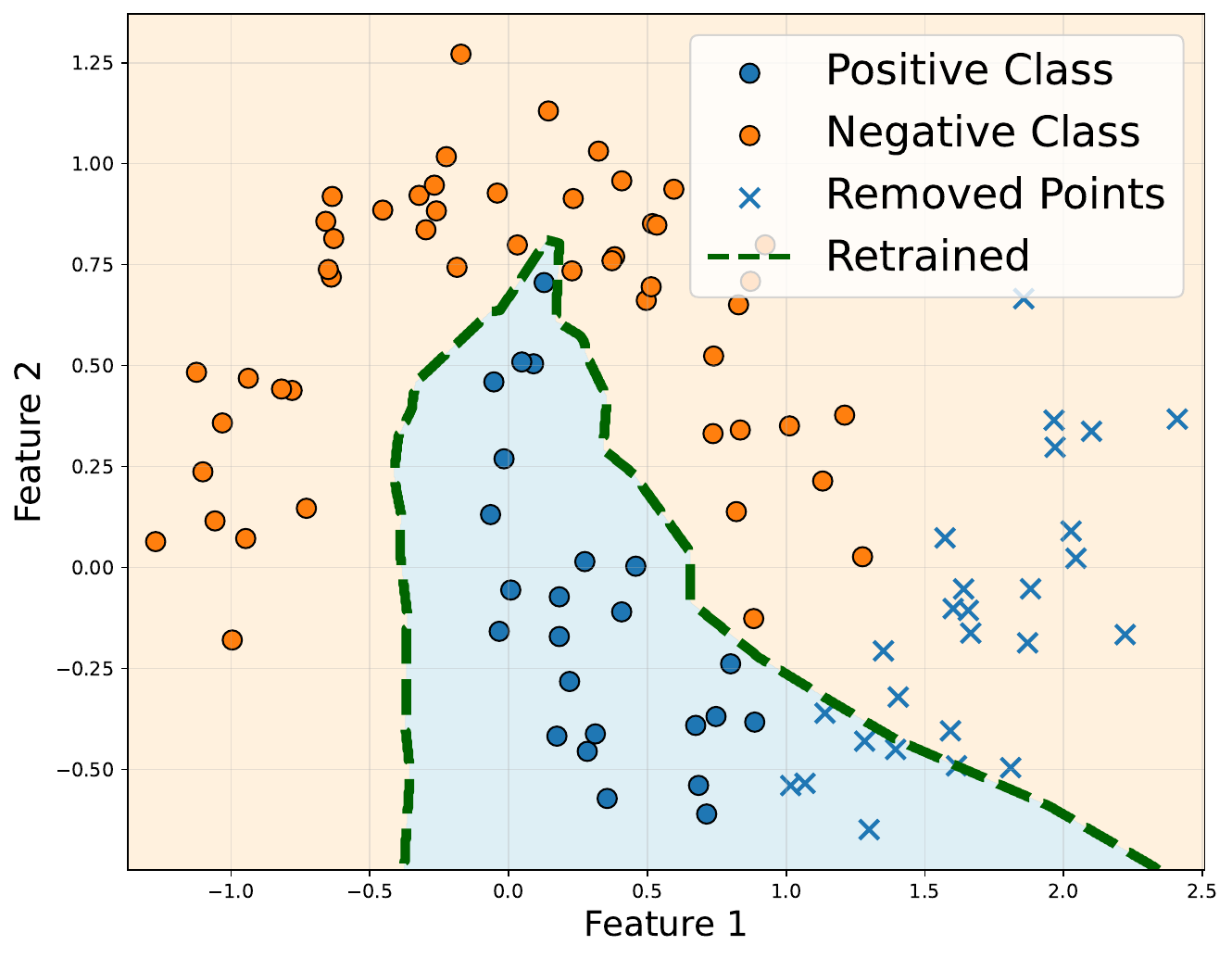} & 
\includegraphics[width=0.31\linewidth]{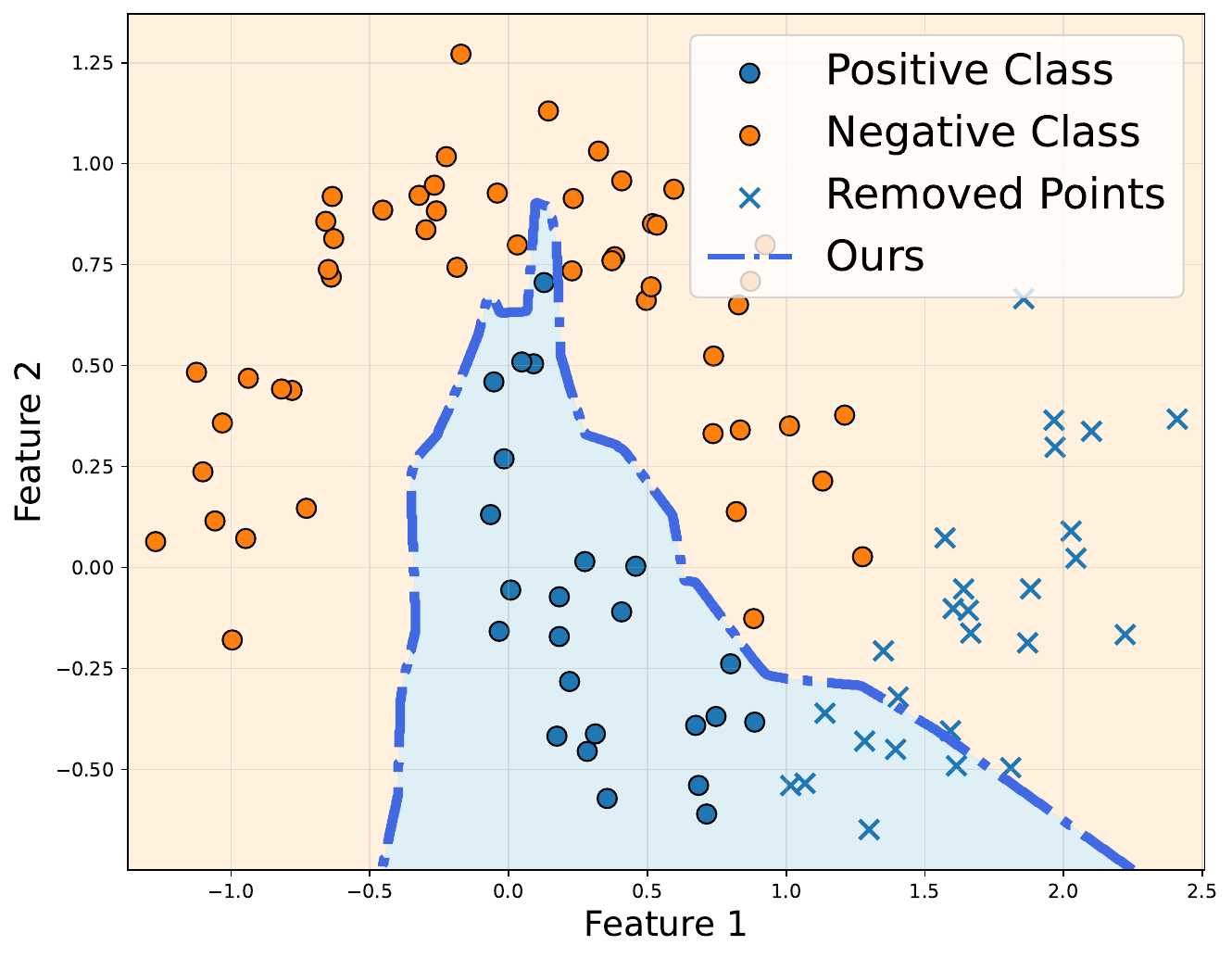}\\
\hspace{0.5cm} (a) Base SPM (Before Unlearning) & \hspace{0.5cm} (b) Unlearned by Retraining & \hspace{0.5cm} (c) Unlearned by Deletion (Ours)\\
\end{tabular}
\vspace{-0.2cm}
\captionof{figure}{{\bf Visualization of semi-parametric classifier's boundaries before and after unlearning.} Points shown in `$\tX$' are removed from the training set. We visualize: (a) the base semi-parametric model (SPM) trained on the full dataset; (b) the unlearned (oracle) SPM by retraining with samples removed; (c) the unlearned SPM by our test-time deletion. We observe that SPM's unlearned decision boundary by our method in (c) closely matches the one from retraining with the datapoints removed in (b).
}
\vspace{-0.1cm}
\label{fig:teaser}

%% file: src/abs.tex
\begin{abstract}
Recent advances in machine unlearning have focused on developing algorithms to remove specific training samples from a trained model. In contrast, we observe that not all models are equally easy to unlearn. Hence, we introduce a family of deep semi-parametric models (SPMs) that exhibit non-parametric behavior during unlearning. SPMs use a fusion module that aggregates information from each training sample, enabling explicit test-time deletion of selected samples without altering model parameters. Empirically, we demonstrate that SPMs achieve competitive task performance to parametric models in image classification and generation, while being significantly more efficient for unlearning. Notably, on ImageNet classification, SPMs reduce the prediction gap relative to a retrained (oracle) baseline by $11\%$ and achieve over $10\times$ faster unlearning compared to existing approaches on parametric models. The code is available at \url{https://github.com/amberyzheng/spm_unlearning}.
\end{abstract}

%% file: src/intro.tex
\vspace{-0.5cm}

\section{Introduction} \label{sec:intro}

With the growing awareness of data privacy, model safety, and regulatory compliance, such as the General Data Protection Regulation (GDPR)~\cite{EU2016GDPR}, there is an increasing interest in the ability to remove the influence of specific samples from a trained model. Machine unlearning (MU) mainly focuses on \textit{developing algorithms} for removal without retraining the model from scratch~\cite{nguyen2020variational,wu2022puma,guo2019certified,sekhari2021remember,neel2021descent}. Given a trained model and a set of samples to be unlearned, an MU algorithm fine-tunes the model as if it had been retrained without including these samples.

While MU has demonstrated promising results, it remains a challenging task due to the black-box nature of deep learning; it is nontrivial to disentangle the contribution of individual training samples to the learned model parameters. However, this is not true for all models. For non-parametric models, such as the $K$-nearest neighbor~\cite{cover1967nearest}, unlearning can be achieved by simply removing the samples from the training set. 
In this work, instead of focusing on \textit{how} to unlearn, we take a different perspective and ask:\\
\indent\textit{Can we {\bf design neural networks} that are fundamentally easier and more efficient to unlearn?} 

We propose Designing to Forget (DTF), a family of deep semi-parametric models (SPMs) suitable for unlearning. The main motivation for SPMs lies in their \textit{explicit} dependencies, where the predictions are directly influenced by specific training data points. In this case, unlearning can be done by deleting dependencies on data points during a forward pass. In~\figref{fig:teaser}, we illustrate an SPM applied to binary classification. We show that removing samples with test-time deletion produces a decision boundary that closely resembles the oracle, which is retrained from scratch without those samples.

Following prior work~\cite{fan2024salun,jia2023model,huang2024unified}, we validate our proposed SPMs on image classification and generation tasks. First, we show that SPMs' performance on these tasks is comparable to parametric models. We then show that SPMs are more efficient and effective at unlearning compared to parametric models with MU algorithms. %

{\bf\noindent Our main contributions are:}
\begin{itemize}[leftmargin=16pt,topsep=0em]
    \setlength{\itemsep}{0.0pt}
    \setlength{\parskip}{2.5pt}
    \item We introduce a design-centric approach to machine unlearning, focusing on neural network architectures rather than algorithms.
    \item We propose  Designing to Forge (DTF), a family of deep semi-parametric models (SPMs) that enable efficient unlearning through test-time deletion.
    \item We empirically show that SPMs match parametric models' performance on classification and generation tasks, while offering improved unlearning capability.
\end{itemize}

%% file: src/rel.tex
\section{Related Work} \label{sec:rel}

{\bf\noindent Semi-parametric models} combine learnable parametric functions with non-parametric components. Classical kernel methods, such as Gaussian processes~\cite{seeger2004gaussian} and deep kernel learning~\cite{wilson2016deep}, allow flexible function approximation with data-dependent inductive bias. 
Recent work has extended semi-parametric principles to deep learning. 
In vision, non-parametric modules have been used for conditional generation~\cite{qi2018semi}, retrieval-augmented image generation~\cite{blattmann2022retrieval,sheynin2023knndiffusion}, and open-set recognition~\cite{yoshihashi2019classification}.  These models benefit from explicit connections to either training or new samples at test time. In tabular domains, architectures such as NPT~\cite{kossen2021selfattention} apply cross-attention across training points to improve classification, leveraging pairwise interactions in a semi-parametric fashion. Different from these works, which aim to improve prediction or generation quality, our goal is to design SPMs that are inherently suitable for {\it unlearning}. 

\myparagraph{Machine unlearning}~\cite{cao2015towards} aims to enable the removal of a user’s private information from a trained model. The goal is to eliminate the influence of specific data samples, often due to privacy, legal, or ethical considerations. Approximate unlearning approaches aim to achieve this by directly modifying the pre-trained model using the data to be erased, \textit{without retraining} from scratch~\cite{nguyen2020variational,wu2022puma,guo2019certified,sekhari2021remember,neel2021descent}. This task has been studied across various domains, including classification~\cite{chen2023boundary,goel2022towards,golatkar2020eternal,fan2024salun,jia2023model}, regression~\cite{mahadevan2021certifiable,tarun2023deep}, and generative modeling~\cite{gandikota2023erasing,fmn,huang2024unified,heng2023selective,fan2024salun,gandikota2024unified}. In the context of text-to-image models, recent work has focused on erasing specific concepts from generative systems. This includes inference-time interventions~\cite{brack2023sega,schramowski2023safe}, fine-tuning-based unlearning of diffusion models~\cite{gandikota2023erasing,kim2023towards,kumari2023ablating}, and direct model editing approaches~\cite{fmn,gandikota2024unified}. More related to this work, \citet{bourtoule2021machine} propose to unlearn by adapting the data. They focus on merging a set of models trained on disjoint data splits and achieve unlearning by deleting \textit{part of the model} trained on the unlearned data.

While prior methods have demonstrated the feasibility of approximate unlearning, they require explicit fine-tuning or model editing on the pre-trained weights, incurring an overhead, particularly in scenarios requiring frequent unlearning. In contrast, we propose SPMs that enable unlearning through test-time deletion on specific datapoints, without any modification to the learned parameters.

\myparagraph{AI safety}
has received increasing attention with rapid progress in generative models~\cite{brundage2018malicious,marchal2024generative,ISRSAA2025,gyevnar2025ai,barez2025open}. A key dimension is data governance: machine unlearning (MU) removes the influence of specific training samples from pre-trained models~\cite{liu2025rethinking,shaik2024exploring,wang2024machine}, motivated by privacy and compliance considerations, including links to differential privacy~\cite{guo2019certified,neel2021descent,sekhari2021remember,ginart2019making,ullah2021machine}. Complementarily, prior work on model immunization~\cite{zheng2024imma,zheng2024learning,zheng2025multi,zheng2025model} develops defenses that harden open-source image classification and diffusion models against harmful adaptations.

%% file: src/prelim.tex
\section{Preliminaries} \label{sec:prelim}
We briefly review the necessary concepts and notation.

\myparagraph{Parametric vs. non-parametric.}
Consider a generic machine learning setup~\cite{murphy2022probabilistic}, given a training set $\gT = \{(\vx, \vy)\}$ with input $\vx$ and ground-truth $\vy$, the goal of \textit{parametric} modeling is to learn the parameters $\theta$ of a model $F_\theta$ that minimizes a loss function $\gL$ of the form
$\min_\theta \gL(\gT, \theta) \triangleq \min_\theta \sum_{(\vx,\vy) \in \gT} \ell(F_\theta(\vx), \vy).$
Here, $\ell$ is a suitable loss function for the task.
At test-time, a parametric prediction 
\bea\label{eq:parametric_pred}
\hat\vy_{\tt p} \triangleq F_{\theta^\star}(\vx)
\eea
relies \textit{solely} on the test sample $\vx$ and the learned parameters $\theta^\star$. This means that all the information from the training set $\gT$ is implicitly encoded within the model parameters $\theta^\star$. Consequently, it is not straightforward to parse out the contribution of an individual training data point $\vx$ from $\theta^\star$.

In contrast, non-parametric models~\cite{wasserman2006all,simonoff2012smoothing} make predictions by explicitly referencing the training set $\gT$, which remains available at test time. Commonly, non-parametric models $H(\vx, \gT)$ utilize a subset of the training data, \ie, a neighborhood $\gN(\vx;\gT)$, based on the test data $\vx$ and aggregate their labels to make a prediction
\bea
\label{eq:nonparametric}
\hat\vy_{\tt np}
\triangleq  H(\vx, \gT) = \sum_{(\vx_i,\vy_i)\in \gN(\vx;\gT)} k(\vx_i, \vx)\,\vy_i,
\eea
where the weights
$k(\vx_i, \vx)\ge 0$, and $\sum_i k(\vx_i, \vx)=1$,
 are determined by a similarity kernel. With the explicit form in~\equref{eq:nonparametric}, the contribution of each training example can be reasoned through $k(\vx_i, \vx)$.
 
\myparagraph{Machine unlearning.} Given a set of samples to be unlearned $\gU$, the oracle unlearned parameters $\theta'$ are obtained by retraining a model without $\gU$, \ie, 
\bea\label{eq:unlearn}
\theta' \triangleq \argmin_\theta \gL(\gT \setminus \gU, \theta).
\eea
In practice, retraining according to~\equref{eq:unlearn} is computationally expensive. Therefore, unlearning algorithms are designed to efficiently update the parameters of a trained model, $\theta^*$ using $\gU$ to produce an approximation $\tilde{\theta}$ of the oracle unlearned parameters. That is, the updated model’s output should closely match that of the oracle model:
\bea\label{eq:unlearn_goal}
F_{\tilde{\theta}}(\vx) \approx F_{\theta'}(\vx), \quad \forall \vx.
\eea
Importantly, the unlearning procedure must be more efficient than retraining from scratch as in~\equref{eq:unlearn}.

%% file: src/app.tex
\input{figs/pipeline}

\section{Designing to Forget} \label{sec:approach}
Our goal is to design effective models that are \textit{inherently suitable for unlearning}. As motivation, unlearning in $K$-nearest neighbors (KNN) is straightforward: a datapoint can be removed by simply excluding it from the set of neighbors at test-time. This simplicity highlights the advantage of non-parametric models in unlearning. However, non-parametric models often fall short in a strong task performance. To address this gap, we propose semi-parametric models (SPMs), which have the ease of unlearning of non-parametric models with the strong task performance of parametric models.

Different from parametric $F_\theta(\vx)$ in~\equref{eq:parametric_pred}, at test-time, a trained SPM $G_{\theta^\star}$ takes in the test data $\vx$, \textit{and} the training dataset $\gT$. An SPM's prediction has the form
\bea\label{eq:semi-parametric}
\hat{\vy} = G_{\theta^\star}(\vx, \gT),
\eea
which explicitly depends on the training data $\gT$, resembling a non-parametric prediction in~\equref{eq:nonparametric}. For clarity, we refer to the test-time provided $\gT$ in~\equref{eq:semi-parametric} as ``the inputted set''.

\myparagraph{Unlearning through test-time deletion.}
As a result, unlearning the set $\gU$ of SPM can be done by deleting that data at test-time:
\bea\label{eq:semi-parametric-unlearn}
\hat{\vy} = G_{\theta^\star}(\vx, \gT \setminus \gU)
\eea
without modifying the trained parameters $\theta^\star$.
{\it There is a caveat}: this approach {\bf does not hold} for all SPMs. Consider an extreme case, where the model $G_\theta$ disregards the inputted set $\gT$, then the SPM prediction in~\equref{eq:semi-parametric} falls back to a parametric model. 

In this work, we propose a set of design principles and effective modules to build SPMs $G_\theta$ that can be unlearned efficiently by using test-time deletion in~\equref{eq:semi-parametric-unlearn}, while matching the competitive performance to existing parametric models. %
In~\secref{sec:spm_unlearn}, we start with an abstract SPM framework, followed by specific instantiations of SPM for classification and generative tasks in~\secref{sec:class} and~\ref{sec:generative}. 

\subsection{SPM architectures for unlearning}
\label{sec:spm_unlearn}
Our proposed deep SPM consists of three modules:
\begin{enumerate}[nosep,label=\alph*), leftmargin=16pt, topsep=0em]
\setlength{\itemsep}{0.0pt}
\setlength{\parskip}{2.5pt}
\item Fusion module $g: {\sR^d \times \sS \rightarrow \sR^d}$ that fuses the parametric and non-parametric branch by mapping a pair consisting of a vector and a set to a vector, where $\sS$ denotes the space of sets $\{\vs_1, \vs_2, \dots \}$ with each $\vs_i \in \sR^d$.
\item Non-parametric module $h: \sS \rightarrow \sS$ that transforms a set into another set. This module processes each of the datapoints in $\gT$ into an ``instance embedding'' for fusion.
\item Parametric module $f: \sR^d \rightarrow \sR^d$ that maps a vector to another vector through standard deep-net layers.
\end{enumerate}
\vspace{1pt}
For readability, we use the same dimensionality $d$ for the inputs and outputs of all layers, and omit explicit notation for learnable parameters. 

Using these three types of module $g$, $h$, and $f$, we construct the semi-parametric model $G_\theta(\vx, \gT)$. An illustration of SPM for a diffusion model is provided in~\figref{fig:pipeline}. The architecture consists of two branches:
An input branch (shown in blue) and a set branch (shown in orange). At each layer in the input branch, features from the set branch are fused to generate latent features $\vz^{(l)}$. At each set branch's layer, it iteratively computes instance features $\gS^{(l)}$ over the input dataset $\gT$. 

Formally, an SPM's forward pass is given by:
\bea\nonumber
&\vz^{(l+1)} = f^{(l)}\left(g^{(l)}(\vz^{(l)}, \gS^{(l)})\right) \quad \text{and} \\ 
\label{eq:semi-forward}
&\gS^{(l+1)} = h^{(l)}(\gS^{(l)}), \quad \forall l \in \{0, \dots, L\},
\eea
where the superscript $(l)$ indicates the layer index. We initialize with $\vz^{(0)} \triangleq \vx$ and $\gS^{(0)} \triangleq \gT$. We now elaborate on the high-level design of each module.

\myparagraph{Fusion module $g$.}
This module connects the parametric and non-parametric components of the model. For unlearning to be effective, the fusion module must meaningfully incorporate the set of instance embeddings $\gS$. To encourage this behavior, we design $g$ to be a weighted combination:
\bea\label{eq:fusion}
g(\vz, \gS) \triangleq \sum_{\vs_i \in \gS \setminus \{\vs_z\}} \alpha(\vz, \vs_i) \cdot \vs_i,
\eea
where $\alpha(\cdot) \in \sR$ is a function that returns a scalar weight between the input latent $\vz$ and  $\vs_i$ instance embedding corresponding to the $i$-th input $\vx_i \in \gT$. During training, the latent $\vz$ is extracted from the training set. Without loss of generality, we denote the associated sample as the $z$-th datapoint $\vx_z$. We \textbf{exclude} the corresponding instance embedding $\vs_z$ from the summation in~\equref{eq:fusion}.

The intuition is to encourage the model to learn features \textit{relative to} the rest of the dataset, similar to the spirit of non-parametric methods. This prevents the model from ignoring other datapoints and relying solely on the parametric modules. Additionally, the fusion in~\equref{eq:fusion} is permutation invariant because summation is orderless, \ie, permuting the indices of $\vs_i \in \gS$ yields the same outcome, which improves data efficiency during training.

\myparagraph{Non-parametric module $h$.} This module transforms the input training set $\gT$ into instance embeddings $\gS^{(l)}$ to preserve non-parametric behavior. In~\equref{eq:semi-forward}, $h^{(l)}$ denotes the $l$-th set-branch layer, implemented as a \emph{shared} instance-wise transformation on $\gS^{(l)}$ for efficiency and permutation equivariance. The module $h$ is trained end-to-end with $f$ and $g$: it encodes samples in $\gT$ into instance features that are fused with parametric representations through $g$. The main design requirement for this module is permutation equivariance. That is, a permutation on the input elements would result in the same permutation at the output, formally,
\bea
h([\vs_{\pi(1)},\dots, \vs_{\pi(|\gS|)}]) = [h_{\pi(1)}, \dots, h_{\pi(|\gS|)} ],
\eea
where $\pi$ denotes any permutation of the element's indices.

Architectures that are permutation equivariant are well studied, such as deep sets~\cite{zaheer2017deep}, graph networks~\cite{battaglia2018relational} (with a complete graph), transformer-based architectures~\cite{vaswani2017attention,lee2019set} (treating each datapoint as a token). For a large dataset, maintaining this non-parametric module is expensive. Hence, we propose a clustering or retrieval-based method to reduce the size of the set $\gS$. The exact architecture and reduction method are task dependent, which we defer to~\secref{sec:class} and~\secref{sec:generative}.

\myparagraph{Parametric module $f$.} For this module, any existing parametric deep-net can be used. In fact, we deliberately design SPM~\equref{eq:semi-forward} in a two-branch structure, such that pre-trained parametric models can be adapted into SPMs.

\myparagraph{Label-permutation augmentation.}
Despite the careful design, it is still possible for SPM to ignore the input training dataset $\gT = \{(\vx,\vy)\}$. For the class-conditioned generation task, the input class label $\vy$ is represented with a one-hot vector. We find that the model would ignore $\vx$ and use the one-hot vectors as a ``bias'' term to turn the non-parametric branch into a parametric one. To prevent such behavior, we shuffle the index of the one-hot vector during training. At each gradient update, a randomly generated permutation matrix is applied to $\vy$ for all data. This forces the model to use the image $\vx$ from the training set, as just $\vy$ alone does not contain enough information.

\input{src/spm_class}
\input{tabs/std_cls}
\input{src/spm_gen}

%% file: figs/pipeline.tex
\begin{figure*}[t]
    \centering
    \includegraphics[width=.82\linewidth]{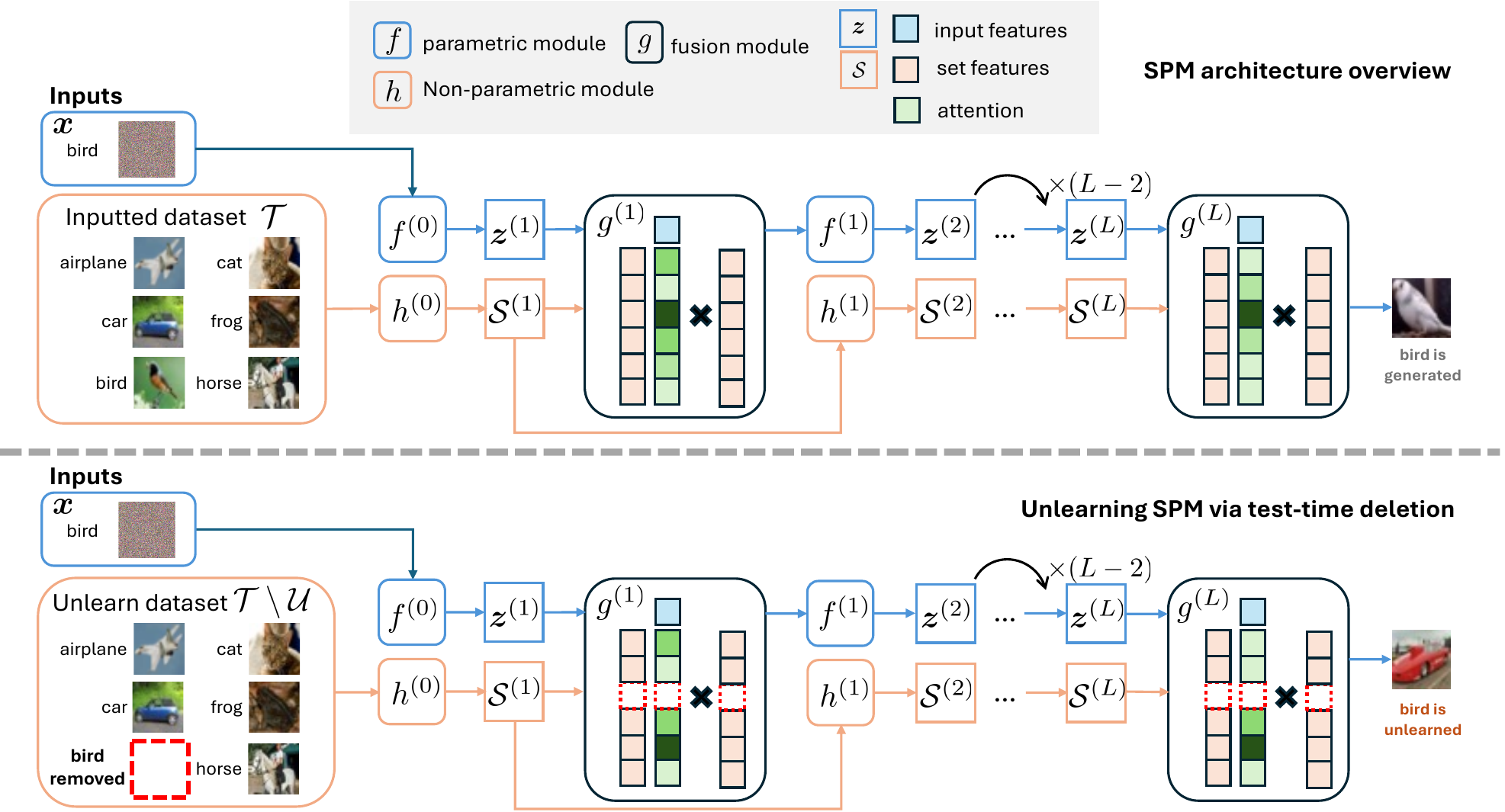}
    \vspace{-0.1cm}
    \caption{{\bf Unlearning with Generative SPM.} It consists of three types of layers: (a) \textit{Fusion module} $g$ connects parametric and non-parametric modules; (b) \textit{Non-parametric module} $h$ maps a set to another set; (c) \textit{Parametric module} $f$ maps a vector to a vector. The first row illustrates a standard pipeline for generating a `bird'. The second row presents unlearning `bird' via test-time deletion.
    }
    \label{fig:pipeline}
    \vspace{-0.4cm}
\end{figure*}

%% file: src/spm_class.tex
\subsection{SPM for image classification}
\label{sec:class}
As motivated, our proposed design allows existing parametric models to be adapted into SPMs via supervised fine-tuning. For image classification, the training set $\gT=\{(\vx_i,\vy_i)\}$ consists of images $\vx_i$ and the corresponding labels $\vy_i$. In our experiment, we use a ResNet~\cite{he2016deep} as the parametric module $h$ to extract the input latent $\vz^{(L)}$, which corresponds to the features from the last block. 

For the non-parametric module, we use the simplest permutation equivariant model, which processes each training sample independently and concatenates with its label, \ie,
\bea\label{eq:nonparam_embedding}
\gS^{(l)} \triangleq \left\{\left[h^{(l)}(\vs_i^{(l)}), \vy_i\right]\right\}_{i=1}^{|\gT|},
\eea
where the first layer $\vs_i^{(0)}\triangleq \vx_i$. 
Since the same $h^{(l)}$ is \textit{shared across} the set elements, this module is permutation equivariant. 
We also implemented a more expressive non-parametric module using a class-aware GNN with multi-head graph attention, which improves CIFAR-10 accuracy from 94.1\% to 94.4\%, highlighting the benefit of structured relational modeling and the generality of our framework.

To make a final prediction, we fuse the set of instance embeddings $\gS^{(L)}$ with the input latent $\vz^{(L)}$, using an attention-like mechanism (excluding its own element):
\bea
\hat{\vy} = \sum_{[\vs_i^{(L)}, \vy_i] \in \gS^{(L)}  \setminus  \{[\vs_z^{(L)}, \vy_z]\} } 
\alpha(\vz^{(L)},\vs_i^{(L)}) \cdot \vy_i,
\eea
where the weights
\bea
\alpha(\vz^{(L)},\vs_i^{(L)}) \triangleq \frac{\exp\left(
(\mW_{\tt q}^{(L)}\vz^{(L)})^\top (\mW_{\tt k}^{(L)}\vs_i^{(L)})\right)}{\sum_{j}\exp\left((\mW_{\tt q}^{(L)}\vz^{(L)})^\top (\mW_{\tt k}^{(L)}\vs_j^{(L)})\right)}
\eea
is similar to a cross-attention layer, where $\vz^{(L)}$ is the query, and $\vs^{(L)}$ is the key.

Lastly, we improve the efficiency of the model by reducing the size of $\gS$ in two ways. First, a retrieval-based approach where we approximate $\gS$ by retrieving the nearest neighbors. Second, a clustering approach where we create ``average instances'' $\bar\vs_c$ by taking an average over all the $\vs_i$ with class label $c$. Both of which reduce the size of $\gS$ to a constant size. 

%% file: tabs/std_cls.tex
\begin{table}[t]
\centering
\caption{{\bf Classification accuracy on CIFAR10 and ImageNet.} We experiment with subsampling 5\%, 10\%, and 100\% of the training data, either using (R)etrival or (C)lustering based method.
}
\vspace{-0.2cm}
\resizebox{0.48\textwidth}{!}{
\begin{tabular}{l|cc|cc}
\specialrule{.15em}{.05em}{.05em}
\multirow{2}{*}{} & \multicolumn{2}{c|}{\textbf{CIFAR10}} & \multicolumn{2}{c}{\textbf{ImageNet}} \\
 &Acc. $\uparrow$ & Run Time (s) $\downarrow$ & Acc. $\uparrow$ & Run Time (s\textbf{)} $\downarrow$\\
\hline
ResNet18 & 94.9 &   0.7& 68.93 &  26.1 \\
\hline
ResNet18-KNN (5\%) & 94.5$_{\pm\text{0.02}}$&   1.1&  61.5$_{\pm\text{0.08}}$& 34.4\\ 
ResNet18-KNN (10\%) & 94.5$_{\pm\text{0.05}}$&  1.1 & 63.4$_{\pm\text{0.23}}$ &34.5 \\ 
ResNet18-KNN (100\%) & 94.5 &   1.5&  66.9  & 63.7\\ \hline
SPM-R (5\%) &   94.1$_{\pm\text{0.02}}$ & 1.0&   52.0$_{\pm\text{0.12}}$ &  33.8\\
SPM-R (10\%) &    94.1$_{\pm\text{0.03}}$ & 1.1&  54.8$_{\pm\text{4.16}}$  & 34.2\\
SPM-R (100\%) &   94.1 &2.1& 59.9    & 61.2 \\
SPM-C (5\%) &   94.5$_{\pm\text{0.02}}$  & 0.7 &   66.1$_{\pm\text{0.07}}$  &  31.3 \\
SPM-C (10\%) &   94.5$_{\pm\text{0.02}}$ &  0.7 &   66.6$_{\pm\text{0.12}}$ & 31.3  \\
SPM-C (100\%) &   94.5  &  0.7 &   67.1  & 31.3   \\
\specialrule{.15em}{.05em}{.05em}
\end{tabular}}

\label{tab:std_cls}
\vspace{-0.5cm}
\end{table}

%% file: src/spm_gen.tex
\subsection{SPM for image generation}
\label{sec:generative}

We consider the class-conditioned image generation task, where the model is given class-label $\vy_i$ to generate an image $\vx_i$ with the respective class. Our SPM design is based on the popular UNet~\cite{unet} architecture used in diffusion models. The architecture consists of a series of down/mid/up blocks. We treat the down and up blocks as the parametric module $f$. Specifically, the down blocks encode the input image into a compact latent representation $\vz^{(l)} = f^{(l-1)}\circ \cdots \circ f^{(1)}(\vx)$, while the up blocks decode the latents to denoise the image, \ie, $\hat \vx = f^{(L)} \circ \cdots \circ f^{(l)}(\vz^{(l+1)})$. 

For the non-parametric module, we followed the same design in~\equref{eq:nonparam_embedding}, which independently processes each datapoint to produce an instance embedding set $\gS^{(l)}$. In our implementation, patch-level features are extracted from the UNet encoder down blocks, and each training instance in $\gT$ is encoded independently to obtain its corresponding patch representations. To integrate the two branches, we replace the mid block of UNet with a fusion module. 

The fusion module's design aims to spatially align the local regions of the parametric feature and the instance embeddings. To do so, we perform the fusion at a patch level. Let $\vz^{(l)}[p] \in \sR^{d'}$ denote the $p$-th patch for latent $\vz^{(l)}$ and $\vs_j^{(l)}[p]$ to denote the $p$-th patch of the $j$-th instance embedding. We denote the set of all instance embedding patches from an image with class $c$ as
\bea
\tilde{\gS}^{(l)}_c \triangleq \bigcup_{j=1}^{|\gS|} \left\{[\vs_j^{(l)}[p], \vy_j=c] \right\}_{p=1}^{P_S},
\eea
where $P_S$ is the number of patches extracted from each $\vs^{(l)}$.
The fused latent patch  $\vz^{(l+1)}[p]$ is given by
\bea
 \sum_{\vs^{(l)}_j[p] \in \gQ}
\alpha\left(\vz^{(l)}[p], \vs^{(l)}_j[p] \right) \cdot \left(\mW_{\tt k}^{(l)} \vs_j^{(l)}[p]\right),
\eea
where 
$\gQ \triangleq \tilde{\gS}_c^{(l)} \setminus \left\{[\vs_z^{(l)}[p],\vy_z] \right\}.
$
Recall, we use $z$ to denote the index of the sample $(\vx_z,\vy_z)$ for which $\vz$ is extracted from. That is, we exclude the sample for which the latent is extracted from the attention.
The weight $\alpha$ follows a variant of Bahdanau attention~\cite{bahdanau2014neural} %
defined as:
\bea
\texttt{Softmax}{\left(\vw^\top \cdot \kappa (
\mW_{\tt q}^{(l)}\vz^{(l)}{[p]} + \mW_{\tt k}^{(l)}\vs^{(l)}_j[p] + \vb^{(l)})\right)},
\eea
where $\texttt{Softmax}$ is normalized over $\gQ$ and $\kappa$ is the ReLU activation, and $\vw$ is a linear layer shared by all the patches.

%% file: src/exp.tex
\input{tabs/std_ddpm_cifar}
\section{Task Performance Experiments} \label{sec:task_exp}

To justify the need for unlearning SPMs, we first demonstrate that SPM's performance in image classification and generation is competitive with that of existing parametric models. The unlearning experiments are presented in~\secref{sec:unlearn_exp}. Implementation details are in Appendix~\secref{sec:supp_implemet_detail}.

\input{tabs/cls_cifar}
\subsection{Image classification}
\myparagraph{Setup \& baselines.}
We conduct experiments on CIFAR-10~\cite{krizhevsky2009learning} and ImageNet-1K~\cite{deng2009imagenet}. To evaluate the models, we report the top-1 accuracy (Acc.) and runtime over the test set.  We consider the baselines of: {\bf (a)} a parametric model (ResNet18~\cite{he2016deep}), {\bf (b)} a non-parametric model of KNN~\cite{cover1967nearest} using features from ResNet18, and {\bf (c)} our SPM built on top of a ResNet18 using the retrieval (SPM-R)/clustering (SPM-C) approach. Finally, we study how the training set's size $|\gT|$ at test-time affects performance.

\myparagraph{Classification results.}
\tabref{tab:std_cls} reports the accuracy and runtime for each of the methods. Compared to the parametric ResNet, we observe that SPM-C achieves competitive performance in both accuracy and runtime. Compared to the non-parametric KNN, SPM-C achieves higher accuracy and a lower runtime. Notably, our method requires only half the runtime of ResNet18-KNN (31.3s vs. 63.7s). 

Next, we observe that using a larger $\gT$ at test-time improves model performance, for both the KNN and our SPMs. This indicates that SPM can utilize the additional training set samples at test-time. Finally, when comparing SPM-R and SPM-C, we observe that clustering achieves lower runtime and improved accuracy. Overall, these results suggest that SPM matches the performance of a parametric model \textbf{and} exhibits similar non-parametric dependencies to the training set at test-time, \ie, SPMs leverage the inputted set $\gT$ to make predictions.

\input{tabs/cls_imagenet}

\subsection{Class-conditioned image generation}
\myparagraph{Setup \& baselines.}
We evaluate SPM on the CIFAR-10 dataset based on DDPM~\cite{ddpm}. We report the Fréchet Inception Distance (FID)~\cite{heusel2017gans}, and the generation time is for 1,000 images for evaluation. We use the same U-Net as in DDPM. For a non-parametric baseline, we compare with the Gaussian Mixture Model (GMM)~\cite{dempster1977maximum}, which shares similar characteristics to KNN but can be directly applied to generative modeling.

\myparagraph{Generation results.} \tabref{tab:std_ddpm_cifar} reports the FID on CIFAR-10. We observe that when the size of the inputted set $\left|\gT\right| = 50$, SPM achieves a lower FID than DDPM while maintaining a reasonable generation speed.
In general, as $\left|\gT\right|$ increases, the FID score consistently decreases but at the cost of a longer runtime. As for GMM, while it has a fast generation time, its FID is much worse, indicating that it fails to capture the complex data distribution of CIFAR-10.

\section{Unlearning Performance Experiments} \label{sec:unlearn_exp}
We strictly follow prior works' evaluation to demonstrate that SPMs can be unlearned more easily and efficiently, via test-time deletion.

\subsection{Unlearning for classification}
Following prior work~\cite{fan2024salun,jia2023model}, we evaluate class-wise unlearning %
using CIFAR-10 and ImageNet-1K. Additional experiments on unlearning random subsets are in Appendix~\secref{sec:random_unlearn}.

\input{src/unlearn_metric}

{\bf\noindent Baselines.} Experiments are conducted on ResNet18, for which we adapted our SPM.
For \textit{parametric} models, we compare eight MU algorithms: \textbf{(1)} gradient ascent (GA)~\cite{thudi2022unrolling}
that reverses training using gradient ascent on unlearned data, \textbf{(2)} fine-tuning (FT)~\cite{warnecke2021machine} that
fine-tunes the model on the remaining data, \textbf{(3)} influence unlearning (IU)~\cite{koh2017understanding} that leverages influence
functions, two boundary unlearning methods, \textbf{(4)} boundary expansion (BE) and \textbf{(5)} boundary shifting (BS)~\cite{chen2023boundary}, \textbf{(6)} $\ell_1$-sparsity~\cite{jia2023model}  that introduces weight sparsity for unlearning, \textbf{(7)} SalUn~\cite{fan2024salun} and \textbf{(8)} MUNBa~\cite{wu2025munba} that aim to erase the influence of unlearned data in models.
For \textit{non-parametric} models, we compare with KNN using ResNet18's last layer embedding. Finally, we report the performance of oracle unlearned models, which are retrained from scratch.

\myparagraph{Results on CIFAR-10.}
\tabref{tab:cls_cifar} reports the results for unlearning one class and five selected classes. We observe that SPM achieves the lowest soft prediction gap, demonstrating strong similarity between the unlearned and oracle (retrained) model.
SPM also yields the smallest $\Delta$UA and $\Delta$RA, along with a minimal $\Delta$TA, indicating its effectiveness in forgetting targeted concepts while preserving performance on the remaining classes. Also, SPM has the lowest unlearning time. Unlike parametric methods that require iterative updates, neither KNN nor SPM modifies the model weights. Hence, unlearning time only involves indexing the data points, which is less than a second.

\myparagraph{Results on ImageNet-1K.}
\tabref{tab:cls_imagenet} presents the results for unlearning on ImageNet-1K. As in CIFAR-10, SPM outperforms all baselines across all metrics. Notably, SPM achieves a low performance gap. SPM maintains near-zero values for $\Delta$UA and $\Delta$RA, and $\Delta$TA. In terms of unlearning time, prior methods often require multiple hours; SPM and KNN are significantly faster, as these models do not require updates in the parameters.

\subsection{Unlearning for generation.}
\myparagraph{Baselines.}
Following~\citet{huang2024unified}, we evaluate class-conditioned unlearning on CIFAR-10. %
Our method is compared against SalUn~\cite{fan2024salun}, along with two other gradient-based approximate unlearning approaches: Selective Amnesia (SA)~\cite{heng2023selective} and SFR-on~\cite{huang2024unified}. 

\myparagraph{Metrics.} Prior works report unlearning time ($\text{Time}_{\tt MU}$) and  performance gap ($\Delta$) between retrained/oracle model, where performance is measured by the following: 
\begin{itemize}
    \item \textit{Unlearning Accuracy (UA):} A trained classifier $C_{\tt p}$ is used to measure the accuracy on the set of images generated by the model for the unlearned class. An image is successfully unlearned if the classifier does {\it not} predict the unlearned class.

    \item \textit{Fréchet Inception Distance ($\text{FID}_{\tt R}$):} Prior work reports the FID on the ${\tt R}$emaining classes of an unlearned model. The FID is computed with respect to the real images from the dataset. Intuitively, an unlearned model should generate high-quality images for the remaining classes.
\end{itemize}
Next, we propose {\it ${\tt O}$racle Fréchet Inception Distance ($\text{FID}_{\tt O}$)}, which measures the FID between the generated images from the unlearned model and those from the oracle model for all the classes. %
This metric quantifies how closely the unlearned model matches the oracle, providing a more direct measure of unlearning effectiveness.

\myparagraph{Quantitative results.} We report class‐conditioned unlearning performance in~\tabref{tab:ddpm_cifar}. We found SPM consistently achieves the smallest $\text{FID}_{\tt O}$, showing it matches the oracle model's behavior.  Some prior works resort to degrading the unlearned classes, thus resulting in a large $\text{FID}_{\tt O}$ (more than 30).  Additionally, SPM maintains a low $\Delta$UA (below 0.1) and $\Delta$FID$_{\tt R}$ (below 0.5), indicating improved unlearning capability and high image quality for the remaining classes. In terms of efficiency, SPM introduces negligible unlearning time, whereas baseline methods require between 23.8 seconds (SFR-on) and over 18,711 seconds (SA). %

\myparagraph{Qualitative results.} \figref{fig:gene_cifar} shows qualitative results for SPM unlearning, where five CIFAR-10 classes are removed at a time. We visualize generated samples for both the unlearned and remaining classes. The samples corresponding to the forgotten classes (\eg, Classes 0 – 4 for Unlearn 0 – 4) display totally different objects, confirming that the model no longer generates images of the removed categories. Notably, the image quality is high for the remaining classes. 
These results demonstrate that the proposed SPM supports precise, class-level forgetting without harming the generation quality of the remaining categories.

\myparagraph{Ablation study of label permutation.}
We conduct an ablation on the proposed label permutation augmentation for SPM. In~\tabref{tab:ddpm_cifar_label_perm}, we report the unlearning results on CIFAR-10 for SPMs trained with and without label permutation. The model trained without label permutation tends to memorize label mappings within the model weights and fails to unlearn, as can be seen from the higher $\Delta$FID$_{\tt O}$, $\Delta$UA and $\Delta$FID$_{\tt R}$. These results demonstrate the efficacy of the proposed label permutation.

%% file: tabs/std_ddpm_cifar.tex
\begin{table}[t]
\centering
\caption{{\bf FID comparisons on CIFAR-10.} $\left|\gT\right|$ is the size of the inputted set. Runtime (RT) is for generating 1,000 images.}
\vspace{-0.2cm}
\resizebox{\linewidth}{!}{
\begin{tabular}{c|c|c|ccccc}
\specialrule{.15em}{.05em}{.05em}
 \multirow{2}{*}&\multirow{2}{*}{\textbf{DDPM}}& \multirow{2}{*}{\textbf{GMM}} & \multicolumn{5}{c}{\textbf{SPM}} 
 \\ & &&\textbf{$\left|\gT\right| = 10$} & \textbf{$\left|\gT\right| = 20$} & \textbf{$\left|\gT\right| = 50$} & \textbf{$\left|\gT\right| = 100$} & \textbf{$\left|\gT\right| = 1024$} \\
\hline
 FID&7.28 &198.42& 7.71 $_{\pm\text{0.20}}$ & 7.30 $_{\pm\text{0.09}}$ & 7.17 $_{\pm\text{0.18}}$ &
   7.09 $_{\pm\text{0.06}}$ & 7.04 $_{\pm\text{0.13}}$  \\
   RT (s) &42&0.42&52&63&101&173&1486 \\
   Rel. RT (\%) &+0\%&-99\%&+23\%&+50\%&+140\%&+311\%&+3404\%
\\
\specialrule{.15em}{.05em}{.05em}
\end{tabular}}

\label{tab:std_ddpm_cifar}
\vspace{-0.5cm}
\end{table}

%% file: tabs/cls_cifar.tex
\begin{table*}[t]
    \setlength{\tabcolsep}{3pt}
\centering
\caption{\textbf{CIFAR-10 class-wise unlearning on classification averaged over 10 classes.} SPMs achieve close performance to the retrained SPM with a fast unlearning time. The UA, RA, and TA values are reported in parentheses with the corresponding gaps.
}
\vspace{-0.2cm}
\resizebox{\textwidth}{!}{ %
\begin{tabular}{c|c|cccccc|cccccc}
\specialrule{.15em}{.05em}{.05em}
    \multirow{2}{*}{\textbf{Backbone}} & \multirow{2}{*}{\textbf{Method}} & \multicolumn{6}{c|}{\textbf{Unlearn 1 Class  (10$\%$)}} & \multicolumn{6}{c}{\textbf{Unlearn 5 Classes  (50$\%$)}} \\
 &  & \textbf{$\text{PG}_{\tt H}$ $\downarrow$} & \textbf{ $\text{PG}_{\tt S}$$\downarrow$} &\textbf{$\Delta$UA $\downarrow$} & \textbf{$\Delta$RA $\downarrow$} & \textbf{$\Delta$TA $\downarrow$} & {$\text{Time}_{\tt MU}$  (s) $\downarrow$}  & \textbf{$\text{PG}_{\tt H}$ $\downarrow$} & \textbf{ $\text{PG}_{\tt S}$$\downarrow$} & \textbf{$\Delta$UA $\downarrow$} & \textbf{$\Delta$RA $\downarrow$} & \textbf{$\Delta$TA $\downarrow$} & {$\text{Time}_{\tt MU}$  (s) $\downarrow$}  \\
\hline
\multirow{9}{*}{\shortstack{ResNet18}} & \cellcolor{retrain} Retrain & \cellcolor{retrain} 0.00 & \cellcolor{retrain} 0.00& \cellcolor{retrain} 0.00 (100.00) & \cellcolor{retrain} 0.00 (100.00) & \cellcolor{retrain} 0.00 (94.76) & \cellcolor{retrain} 2317.6 & \cellcolor{retrain} 0.00 & \cellcolor{retrain} 0.00 &\cellcolor{retrain} 0.00 (100.00) & \cellcolor{retrain} 0.00 (100.00) & \cellcolor{retrain} 0.00 (96.06) & \cellcolor{retrain} 1284.9 \\
 & GA~\cite{thudi2022unrolling}  & 18.48&0.99& 5.80 (94.20) & 5.64 (94.36) & 6.42 (88.34) & 8.90 & 32.62&2.86 & \textbf{0.00} (100.00) & 15.51 (84.49) & 14.23 (81.77) & 41.3 \\
 & FT~\cite{warnecke2021machine}  & 13.11&0.48& 66.25 (33.75) & 0.15 (99.85) & 0.32 (94.44) & 148.7 &26.21 &0.72 & 34.76 (65.24) & 0.06 (99.94) & 0.59 (96.65) & 81.8 \\
 & IU~\cite{koh2017understanding} & 55.93&3.51 & \textbf{0.00} (100.00) & 54.13 (45.87) & 50.70 (44.06) & 17.7 & 75.71&424.85 & \textbf{0.00} (100.00) & 79.97 (20.03) & 76.04 (20.02) & 16.8 \\
 & BE~\cite{chen2023boundary}  & 15.43&0.71& 23.01 (76.99) & 1.81 (98.19) & 2.68 (92.08) & 15.1 & 46.00& 2.11& 64.21 (35.79) & 0.40 (99.60) &\textbf{ 0.17} (95.89) & 81.6 \\
 & BS~\cite{chen2023boundary}  & 15.28&0.71& 23.10 (76.90) & 1.74 (98.26) & 2.59 (92.17) & 28.1  &44.01 &1.95& 62.25 (38.75) & 0.40 (99.60) & 0.04 (96.02) & 159.7 \\
 & $\ell_1$-sparse~\cite{jia2023model}  &13.03 &0.46& 60.26 (39.74) & 0.22 (99.78) & 0.55 (94.21) & 148.8  & 26.12&0.72& 34.94 (65.30) & 0.06 (99.94) & 0.59 (96.65) & 82.8 \\
 & SalUn~\cite{fan2024salun}  &9.76 &0.32& 0.90 (99.10) & 0.06 (99.94) & \textbf{0.14} (94.90) & 137.8  & 25.84&0.74& 1.98 (98.02) & 0.03 (99.97) & 0.82 (96.88) & 142.9 \\
 & MUNBa~\cite{wu2025munba} & 16.92& 0.50& 8.47 (91.53) & 0.22 (99.78) & 0.55 (94.21) & 369.1 & 29.38& 1.46& \textbf{0.00} (91.53) & 0.22 (99.78) & 1.85 (94.21) & 369.1 \\
\hline
\multirow{2}{*}{ResNet18-KNN} & \cellcolor{retrain}Retrain &\cellcolor{retrain} 0.00 & \cellcolor{retrain} 0.00& \cellcolor{retrain}0.00 (100.00) & \cellcolor{retrain}0.00 (100.00) &\cellcolor{retrain} 0.00 (93.22) &\cellcolor{retrain} 2317.6  & \cellcolor{retrain} 0.00 & \cellcolor{retrain} 0.00&\cellcolor{retrain} 0.00 (100.00) & 0.00 (99.59) \cellcolor{retrain}& 0.00 (96.07) \cellcolor{retrain}&\cellcolor{retrain}1284.9\\
&Deletion &\textbf{8.10}& 0.55 & \textbf{0.00} (100.00) & 0.44 (99.56) & 1.77 (94.99) & \textbf{$<$1} & \textbf{17.81}&0.67 & \textbf{0.00} (100.00) & 0.13 (99.72) & 0.88 (96.95) &\textbf{$<$1}\\
\hline
ResNet18-SPM 
& \cellcolor{retrain} Retrain & \cellcolor{retrain} 0.00 & \cellcolor{retrain} 0.00 &\cellcolor{retrain} 0.00 (100.00) & \cellcolor{retrain} 0.00 (99.98) & \cellcolor{retrain} 0.00 (94.38) & \cellcolor{retrain} 3288.9 &\cellcolor{retrain} 0.00 & \cellcolor{retrain} 0.00 & \cellcolor{retrain} 0.00 (100.00) & \cellcolor{retrain} 0.00 (99.99) & \cellcolor{retrain} 0.00 (95.73) & \cellcolor{retrain} 1841.6 \\
(Ours)& Deletion &8.62& \textbf{0.27 }& \textbf{0.00} (100.00) & \textbf{0.00} (99.98) & 0.52 (94.90) & \textbf{$<$1} & 21.20&\textbf{0.52} & \textbf{0.00} (100.00) & \textbf{0.01 }(99.98) & 1.23 (96.96) &\textbf{$<$1} \\
\specialrule{.15em}{.05em}{.05em}
\end{tabular}
}

\label{tab:cls_cifar}
\vspace{-0.5cm}

\end{table*}

%% file: tabs/cls_imagenet.tex
\begin{table*}[t]
\small
\centering
\caption{\textbf{ImageNet 1-class unlearning on classification averaged over 10 classes.} We observe that SPMs achieve close performance to the retrained model with a low unlearning time. %
}
\vspace{-0.2cm}
\resizebox{0.8\textwidth}{!}{ %
\begin{tabular}{c|c|cccccc}
\specialrule{.15em}{.05em}{.05em}
{\textbf{Backbone}} & {\textbf{Method}} 
&  \textbf{$\text{PG}_{\tt H}$ $\downarrow$} & \textbf{ $\text{PG}_{\tt S}$$\downarrow$} & \textbf{$\Delta$UA $\downarrow$} & \textbf{$\Delta$RA $\downarrow$} & \textbf{$\Delta$TA $\downarrow$} & $\text{Time}_{\tt MU}$   (s) $\downarrow$ \\
\hline
\multirow{10}{*}{\shortstack{ResNet18 }} 
& \cellcolor{retrain} Retrain & \cellcolor{retrain} 0.00 & \cellcolor{retrain} 0.00&\cellcolor{retrain}0.00 (100.00) & \cellcolor{retrain}0.00 (68.19) & \cellcolor{retrain}0.00 (68.13) & \cellcolor{retrain}226025.4  \\
& GA~\cite{thudi2022unrolling} & 30.47&0.78& 7.31 (92.69) & 3.41 (64.78) & 2.72 (65.41) & 24.8  \\
& FT~\cite{warnecke2021machine} & 29.68&0.53 & 0.01 (99.99) & 5.68 (62.51) & 3.84 (64.29) & 11361.4  \\
& IU~\cite{koh2017understanding} &39.18 &1.10 & 4.15 (95.85) & 13.71 (54.48) & 10.85 (57.28) & 5684.0  \\
& BE~\cite{chen2023boundary} &33.94& 0.92 & 3.96 (96.04) & 7.17 (61.02) & 5.79 (62.34) & 47.0  \\
& BS~\cite{chen2023boundary} &35.45 &0.99 & 1.99 (98.01) & 8.84 (59.35) & 7.18 (60.95) & 44.5  \\
& $\ell_1$-sparse~\cite{jia2023model} &31.29&0.57 & \textbf{0.00} (100.00) & 7.99 (60.19) & 5.53 (62.60) & 11236.7  \\
& SalUn~\cite{fan2024salun} &25.72 &0.42 & 22.43 (77.57) &\textbf{ 0.19 }(68.00) & 0.39 (68.52) & 225519.2  \\
& MUNBa~\cite{wu2025munba} &34.05 &0.65 & \textbf{0.00} (100.00) & 11.64 (56.55) & 8.45 (59.68) & 16685.9  \\
\hline

\multirow{2}{*}{ResNet18-KNN} & \cellcolor{retrain}Retrain &\cellcolor{retrain} 0.00 & \cellcolor{retrain} 0.00& \cellcolor{retrain}0.00 (100.00) & \cellcolor{retrain}0.00 (65.94) & \cellcolor{retrain}0.00 (66.15) &\cellcolor{retrain}226025.4 \\
& Deletion & 18.99&2.54 & \textbf{0.00} (100.00) & 2.72 (63.22) & 1.71 (64.44) & \textbf{$<$1}
\\
\hline
\shortstack{ResNet18-SPM}
& \cellcolor{retrain} Retrain& \cellcolor{retrain} 0.00 & \cellcolor{retrain} 0.00& \cellcolor{retrain}0.00 (100.00) & \cellcolor{retrain}0.00 (65.20) & \cellcolor{retrain}0.00 (67.10) & \cellcolor{retrain} 268803.5\\
(Ours)& Deletion & \textbf{14.60}&\textbf{0.12} & \textbf{0.00} (100.00) & \textbf{0.19} (65.01) & \textbf{0.00} (67.10) & \textbf{$<$1}  \\
\specialrule{.15em}{.05em}{.05em}
\end{tabular}
}

\label{tab:cls_imagenet}
\vspace{-0.25cm}
\end{table*}

%% file: src/unlearn_metric.tex
\input{tabs/ddpm_cifar}
{\noindent\bf Existing metrics.} Prior work~\cite{fan2024salun,jia2023model} evaluates an MU algorithm in two dimensions, its efficiency, and unlearning quality. For efficiency, the wall time for each MU algorithm ($\text{Time}_{\tt MU}$) is reported. For quality, prior works report the accuracy gap
\bea\label{eq:acc_gap}
\Delta \text{Acc}(\gD) \triangleq |\text{Acc}(\tilde{F},\gD) - \text{Acc}(F',\gD)| 
\eea
between the unlearned model $\tilde{F}$ and the oracle $F'$ (\equref{eq:unlearn}) over a dataset $\gD$. Prior works considered~\equref{eq:acc_gap} over three dataset splits, leading to the following metrics (lower the better):
\begin{itemize}
    \item $\Delta$UA evaluates the {\bf U}nlearned set, \ie, $\gD \triangleq \gU$.
    \item $\Delta$RA evaluates the {\bf R}emaining set, \ie, $\gD \triangleq \gT \setminus \gU$.
    \item $\Delta$TA evaluates the {\bf T}est set $\gV$.
\end{itemize}

\input{figs/gene_cifar}
\input{tabs/ddpm_cifar_label_perm}

\myparagraph{Rethinking unlearning metrics.} While the accuracy gap in~\equref{eq:acc_gap} is one way to measure the difference between the unlearned and oracle model, a small accuracy gap does not always mean the unlearned/oracle models are similar in all their \textit{predictions}. The models can make different mistakes on a dataset yet achieve the same accuracy.

To address this, we propose {\tt H}ard/{\tt S}oft \textit{{\bf P}rediction {\bf G}aps} for evaluation, which directly measure the difference in prediction on the test-set $\gV$ defined as
    \bea\label{eq:oracle-agree}
    \text{PG}_{\tt H}(\tilde{F},F')
    = \frac{1}{|\gV|}\sum_{\vx\in\gV}
    \1 [\tilde{\vy}(\vx) \neq \vy'(\vx)],
    \eea
where  
$\tilde{\vy}=\argmax_y \tilde{F}(\vx)[y]$ is the hard prediction from the unlearned model $\tilde{F}$. Similarly, $\vy'$ is the prediction from the oracle. Here, $\mathbf{1}[\cdot]$ denotes the indicator function.

For soft-predictions where $\tilde{F}$ and $F'$ output the class probabilities, we measure the difference using KL divergence, \ie,
\bea
\label{eq:oracle-kl}
 \text{PG}_{\tt S}(\tilde{F},F')
= \tfrac{1}{|\gV|}\sum_{\vx\in\gV}
\KL\!\big(\tilde{F}(\vx), F'(\vx) \big).
\eea

We note that the proposed prediction gap is stricter than the accuracy gap in~\equref{eq:acc_gap}, \ie, an upper bound. 
\begin{restatable}[]{claim}{claimone}
\label{thm:metric}
Given that the oracle margin $\gamma(\vx)=F'(\vx)[y_1]-F'(\vx)[y_2]\ge \gamma_{\min}>0$ for all $\vx\in\gV$, where $y_1,y_2$ denote the indices of the largest and second predicted probabilities from the oracle model.
We have
\bea\label{eq:metric-bound}
\Delta \text{Acc}({\gV}) \leq \text{PG}_{\tt H}(\tilde{F}, F')
\;\le\; \frac{\sqrt{2}}{\gamma_{\min}}\sqrt{\text{PG}_{\tt S}(\tilde{F},F')}. 
\eea
\end{restatable}
\begin{proof} The proof is based on the triangle inequality of total variation; We defer it to Appendix~\secref{sec:supp_proof}.
\end{proof}

%% file: tabs/ddpm_cifar.tex
\begin{table*}[t]

\centering
\caption{\textbf{CIFAR-10 class-wise forgetting on generation.} We experiment on 5 unlearning scenarios with unlearning classes: \textit{Automobile}, \textit{Cat}, \textit{Dog}, \textit{Horse}, and \textit{Truck}. We observe that our method achieves close performance to the retrained model with a low unlearning time.}
\vspace{-0.2cm}
\renewcommand{\arraystretch}{1.2}
\setlength{\tabcolsep}{2pt}
\small
\resizebox{\textwidth}{!}{
\begin{tabular}{c|c|ccc|ccc|ccc|ccc|ccc|c}
\specialrule{.15em}{.05em}{.05em}
\multirow{2}{*}{\textbf{Backbone}} & \multirow{2}{*}{\textbf{Method}} 
& \multicolumn{3}{c|}{\textbf{Automobile}} 
& \multicolumn{3}{c|}{\textbf{Cat}} 
& \multicolumn{3}{c|}{\textbf{Dog}} 
& \multicolumn{3}{c|}{\textbf{Horse}} 
& \multicolumn{3}{c|}{\textbf{Truck}} 
& \multirow{2}{*}{{$\text{Time}_{\tt MU}$   (s) $\downarrow$}} \\
& 
& \textbf{FID$_{\tt O}$ $\downarrow$} & \textbf{$\Delta$UA $\downarrow$} & \textbf{$\Delta$FID$_{\tt R}$ $\downarrow$} 
& \textbf{FID$_{\tt O}$ $\downarrow$} & \textbf{$\Delta$UA $\downarrow$} & \textbf{$\Delta$FID$_{\tt R}$ $\downarrow$} 
& \textbf{FID$_{\tt O}$ $\downarrow$} & \textbf{$\Delta$UA $\downarrow$} & \textbf{$\Delta$FID$_{\tt R}$ $\downarrow$} 
& \textbf{FID$_{\tt O}$ $\downarrow$} & \textbf{$\Delta$UA $\downarrow$} & \textbf{$\Delta$FID$_{\tt R}$ $\downarrow$} 
& \textbf{FID$_{\tt O}$ $\downarrow$} & \textbf{$\Delta$UA $\downarrow$} & \textbf{$\Delta$FID$_{\tt R}$ $\downarrow$} \\
\hline

\multirow{4}{*}{\shortstack{DDPM}} & \cellcolor{retrain} Retrain &  \cellcolor{retrain}0.00 &  \cellcolor{retrain} 0.00 (100.00) & 0 \cellcolor{retrain}.00 (11.21) &  \cellcolor{retrain}0.00 &  \cellcolor{retrain}0.00 (99.98) & \cellcolor{retrain} 0.00 (10.84) &  \cellcolor{retrain}0.00&  \cellcolor{retrain}0.00 (100.00) &  \cellcolor{retrain}0.00 (10.88) &  \cellcolor{retrain}0.00 &  \cellcolor{retrain}0.00 (99.98) &  \cellcolor{retrain}0.00 (10.00) &  \cellcolor{retrain}0.00 &  \cellcolor{retrain}0.00 (99.90) &  \cellcolor{retrain}0.00 (10.15) &  \cellcolor{retrain}145601.4 \\ 
& SA~\cite{heng2023selective} & 31.00 & \textbf{0.00 }(100.00) & 12.35 (23.56) & 32.54 & 14.18 (85.80) & 10.50 (21.34) & 32.03 & 8.60 (91.40) & 11.11 (21.19) & 477.61 & \textbf{0.02} (100.00) & 11.13 (21.13) & 35.79 & 0.10 (100.00) & 18.89 (29.04) & 18711.0 \\
& SalUn~\cite{fan2024salun} & 6.30 & 0.20 (99.80) & 10.02 (21.23) & 6.88 & 1.38 (98.60) & 9.45 (20.29) & 6.93 & \textbf{0.00} (100.00) & 9.30 (20.18) & 6.11 & 0.58 (99.40) & 10.70 (20.70) & 3.59 & 0.70 (99.20) & 10.30 (20.45) & 493.2 \\
& SFR-on~\cite{huang2024unified} & 25.69& \textbf{0.00 }(100.00) & 9.49 (20.70) & 41.68 & 7.38 (92.60) & 7.60 (18.44) & 41.30 & 0.20 (99.80) & 8.01 (18.89) & 91.48& \textbf{0.02} (100.00) & 9.93 (19.93) & 42.70 & 0.10 (100.00) & 10.46 (20.61) & 23.8 \\
\hline
\multirow{2}{*}{\shortstack{SPM}} & \cellcolor{retrain} Retrain &  \cellcolor{retrain}0.00&  \cellcolor{retrain}0.00 (99.76) & \cellcolor{retrain} 0.00 (7.40) & \cellcolor{retrain} 0.00 &  \cellcolor{retrain}0.00 (99.26) &  \cellcolor{retrain}0.00 (7.15) & \cellcolor{retrain} 0.00 & \cellcolor{retrain} 0.00 (99.78)& \cellcolor{retrain} 0.00 (6.52) &  \cellcolor{retrain}0.00 &  \cellcolor{retrain}0.00 (99.84) &  \cellcolor{retrain}0.00 (6.89) & \cellcolor{retrain} 0.00 &  \cellcolor{retrain}0.00 (99.64)& \cellcolor{retrain} 0.00 (7.28) & \cellcolor{retrain} 90935.9 \\
& Ours & \textbf{1.97}& 0.10 (99.86) & \textbf{0.38} (7.78) &\textbf{ 1.85} & \textbf{0.04} (99.22) & \textbf{0.50 }(7.65) & \textbf{1.80 }& 0.04 (99.82) & \textbf{0.08} (6.60) & \textbf{1.90 }& 0.08 (99.76) & \textbf{0.29} (7.18) & \textbf{1.97} & \textbf{0.06 }(99.58) & \textbf{0.40} (7.68) &\textbf{$<$1} \\
\specialrule{.15em}{.05em}{.05em}
\end{tabular}
}
\label{tab:ddpm_cifar}
\vspace{-0.45cm}
\end{table*}

%% file: figs/gene_cifar.tex
\begin{figure*}[t]
\small
\centering
\renewcommand{\arraystretch}{1.2}
\setlength{\tabcolsep}{2pt}
\begin{tabular}{>{\centering\arraybackslash}m{1.8cm}|*{5}{>{\centering\arraybackslash}m{0.07\textwidth}}|*{5}{>{\centering\arraybackslash}m{0.07\textwidth}}}
\specialrule{.15em}{.05em}{.05em}
  & \multicolumn{5}{c}{\textbf{Class 0 - 4}} & \multicolumn{5}{c}{\textbf{Class 5 - 9}} \\
& planes & cars & birds & cats & deer & dogs & frogs & horses & ships & trucks \\
\midrule
\makecell{Pre-trained\\SPM}
& \includegraphics[width=0.05\textwidth]{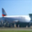} &
  \includegraphics[width=0.05\textwidth]{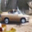} &
  \includegraphics[width=0.05\textwidth]{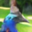} &
    \begin{overpic}[width=0.05\textwidth]{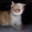} 
\linethickness{0.8pt}
    \put(0,0){\color{blue}\framebox(100,100){}}
\end{overpic} &
  \includegraphics[width=0.05\textwidth]{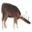} &
  \includegraphics[width=0.05\textwidth]{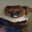} &
    \begin{overpic}[width=0.05\textwidth]{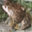} 
\linethickness{0.8pt}
    \put(0,0){\color{blue}\framebox(100,100){}}
\end{overpic} &
  \includegraphics[width=0.05\textwidth]{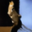} &
  \includegraphics[width=0.05\textwidth]{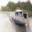} &
  \includegraphics[width=0.05\textwidth]{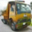} \\
\midrule
\makecell{Unlearn \\ 0 - 4} &
  \includegraphics[width=0.05\textwidth]{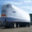} &
  \includegraphics[width=0.05\textwidth]{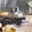} &
  \includegraphics[width=0.05\textwidth]{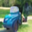} &
\begin{overpic}[width=0.05\textwidth]{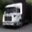}
    \linethickness{0.8pt}
    \put(0,0){\color{red}\framebox(100,100){}}
\end{overpic} &
  \includegraphics[width=0.05\textwidth]{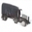} &
  \includegraphics[width=0.05\textwidth]{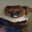} &
  \begin{overpic}[width=0.05\textwidth]{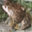} 
\linethickness{0.8pt}
    \put(0,0){\color{blue}\framebox(100,100){}}
\end{overpic} &
  \includegraphics[width=0.05\textwidth]{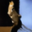} &
  \includegraphics[width=0.05\textwidth]{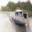} &
  \includegraphics[width=0.05\textwidth]{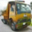} \\
\midrule
\makecell{Unlearn \\ 5 - 9} &
  \includegraphics[width=0.05\textwidth]{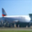} &
  \includegraphics[width=0.05\textwidth]{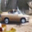} &
  \includegraphics[width=0.05\textwidth]{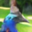} &
      \begin{overpic}[width=0.05\textwidth]{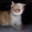} 
\linethickness{0.8pt}
    \put(0,0){\color{blue}\framebox(100,100){}}
\end{overpic} &\includegraphics[width=0.05\textwidth]{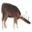} &
  \includegraphics[width=0.05\textwidth]{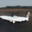} &
  \begin{overpic}[width=0.05\textwidth]{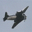}
    \linethickness{0.8pt}
    \put(0,0){\color{red}\framebox(100,100){}}
\end{overpic} &
  \includegraphics[width=0.05\textwidth]{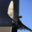} &
  \includegraphics[width=0.05\textwidth]{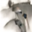} &
  \includegraphics[width=0.05\textwidth]{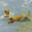} \\
\specialrule{.15em}{.05em}{.05em}
\end{tabular}
\vspace{-0.2cm}
\caption{{\bf Qualitative results for unlearned SPM}. %
When a class is unlearned (\eg, cats or frogs framed in \textcolor{red}{red}), the model avoids generating that concept and produces samples resembling remaining classes (\eg, a truck replacing a cat or a plane replacing a frog). Additionally, the generations from the remaining classes are left unchanged (framed in \textcolor{blue}{blue}).}
\label{fig:gene_cifar}
\vspace{-.17cm}
\end{figure*}

%% file: tabs/ddpm_cifar_label_perm.tex
\begin{table*}[t]
\centering
\caption{\textbf{Ablation study of label permutation.} We experiment on CIFAR-10 class-wise forgetting on generation with SPMs, and the results show that removing label permutation leads to a large gap in UA.}
\vspace{-0.2cm}
\renewcommand{\arraystretch}{1.2}
\setlength{\tabcolsep}{4pt}
\resizebox{\textwidth}{!}{
\begin{tabular}{c|ccc|ccc|ccc|ccc|ccc}
\specialrule{.15em}{.05em}{.05em}
\multirow{2}{*}{\textbf{Method}} 
& \multicolumn{3}{c|}{\textbf{Automobile}} 
& \multicolumn{3}{c|}{\textbf{Cat}} 
& \multicolumn{3}{c|}{\textbf{Dog}} 
& \multicolumn{3}{c|}{\textbf{Horse}} 
& \multicolumn{3}{c}{\textbf{Truck}} 
\\
&  \textbf{FID$_{\tt O}$ $\downarrow$} &\textbf{$\Delta$UA $\downarrow$}   & \textbf{$\Delta \text{FID}_{\tt R}$ $\downarrow$} 
&\textbf{FID$_{\tt O}$ $\downarrow$} &  \textbf{$\Delta$UA $\downarrow$}   & \textbf{$\Delta \text{FID}_{\tt R}$ $\downarrow$} 
&\textbf{FID$_{\tt O}$ $\downarrow$} &  \textbf{$\Delta$UA $\downarrow$}   & \textbf{$\Delta \text{FID}_{\tt R}$ $\downarrow$} 
&\textbf{FID$_{\tt O}$ $\downarrow$} &  \textbf{$\Delta$UA $\downarrow$}   & \textbf{$\Delta \text{FID}_{\tt R}$ $\downarrow$} 
&\textbf{FID$_{\tt O}$ $\downarrow$} &  \textbf{$\Delta$UA $\downarrow$}   & \textbf{$\Delta \text{FID}_{\tt R}$ $\downarrow$}
\\
\hline

\cellcolor{retrain} Retrain 
&\cellcolor{retrain} 0.00  &\cellcolor{retrain} 0.00 (99.76) & \cellcolor{retrain} 0.00 (7.40)
&\cellcolor{retrain} 0.00 & \cellcolor{retrain} 0.00 (99.26) & \cellcolor{retrain} 0.00 (7.15)
&\cellcolor{retrain} 0.00 & \cellcolor{retrain} 0.00 (99.78)& \cellcolor{retrain} 0.00 (6.52)
&\cellcolor{retrain} 0.00& \cellcolor{retrain} 0.00 (99.84) & \cellcolor{retrain} 0.00 (6.89)
&\cellcolor{retrain} 0.00& \cellcolor{retrain} 0.00 (99.64)& \cellcolor{retrain} 0.00 (7.28) 
\\

w/o perm. 
& 2.88 & 0.82 (98.94) & 2.66 (4.74)
&5.08& 3.13 (96.13) & 2.68 (4.47)
&3.39& 83.61 (16.17) & 2.15 (4.37)
&3.87& 60.89 (38.95) & 2.22 (4.67)
&2.90 & 1.56 (98.08) & 2.58 (4.70)
\\

Ours
& \textbf{1.97}& \textbf{0.10 }(99.86) &\textbf{ 0.38 }(7.78)
&\textbf{ 1.85} & \textbf{0.04} (99.22) &\textbf{0.50 }(7.65)
& \textbf{1.80 }& \textbf{0.04} (99.82) & \textbf{0.08} (6.60)
& \textbf{1.90 }& \textbf{0.08} (99.76) &\textbf{ 0.29} (7.18)
& \textbf{1.97}& \textbf{ 0.06 }(99.58) & \textbf{0.40} (7.68)
\\
\specialrule{.15em}{.05em}{.05em}
\end{tabular}
}

\label{tab:ddpm_cifar_label_perm}
\vspace{-0.4cm}
\end{table*}

%% file: src/conc.tex
\section{Conclusion} \label{sec:conc}
We introduced DTF, a family of SPMs for unlearning. The proposed SPMs enable unlearning through test-time deletion of samples from an inputted set, without modifying trained parameters. Empirical results on image classification and generation demonstrate that SPMs achieve competitive performance compared to existing parametric models, while enabling fast and faithful unlearning that closely matches a fully retrained oracle. We hope this work inspires future research on machine unlearning beyond post-hoc methods and encourages model designs suitable for unlearning. We believe SPMs are a viable path towards privacy-compliant AI systems where support for rapid data removal requests is required. 

\vspace{3pt}
{\small
{\noindent\bf Acknowledgements: } This project is supported in part by an NSF Award \#2420724, and a Google Research Scholar Award.
}

%% file: src/supp.tex
\onecolumn
\appendix
\label{sec:appendix}

{\noindent \bf \Large Appendix}

\vspace{6pt}
{\bf\noindent The appendix is organized as follows:}
\begin{itemize}
    \item Implementation and experimental details are presented in~\secref{sec:supp_implemet_detail}.
      \item The complete proof of Claim~\ref{thm:metric} is provided in~\secref{sec:supp_proof}.
    \item Additional qualitative results for GMM generation are included in~\secref{sec:supp_gmm}.
    \item Additional results on random unlearning for classification are reported in~\secref{sec:random_unlearn}.
\end{itemize}

\section{Implementation and Experiment Details}
\label{sec:supp_implemet_detail}

\myparagraph{Details for task performance experiments.}
For SPM on classification, ResNet18 is trained using SGD with a learning rate of 0.1 and a batch size of 256. We train from scratch for 100 epochs on CIFAR-10 and fine-tune the ImageNet-pretrained model for 160 epochs on ImageNet. 
For each mini-batch, we pair each query with 1,024 randomly sampled instances from the training set, which serve as the inputted set during training. All experiments are repeated with five random seeds.

As for SPM on generation, we train a DDPM for 2,000 epochs with a learning rate of 0.0005 and a batch size of 256.
During training, each batch is paired with an inputted set of 256 samples drawn from the training set, serving as the non-parametric memory for SPM.
During evaluation, we generate 5,000 samples per class using DDIM~\cite{ddim} with 50 sampling steps.
Runtime is measured based on the time required to generate 1,000 images. 
All the experiments are conducted on Nvidia L40S.

\myparagraph{Baselines for task performance.}
For classification, we adopt ResNet18 and ResNet18-KNN as the parametric and non-parametric baselines, respectively. We train ResNet18 by following the implementation details provided by~\citet{fan2024salun}.  ResNet18-KNN performs $K$-nearest neighbor classification on the embeddings extracted from ResNet18, with the optimal value of $K$ determined to be 50 for achieving the highest accuracy. In both ResNet18-KNN and SPM-R, the samples from the inputted set with the minimum L2 distance to the query are retrieved, with a retrieval size set of 256. This retrieval process is implemented using the FAISS (Facebook AI Similarity Search) library \cite{douze2025faiss}, which enables efficient similarity search through product quantization (PQ) codes, allowing for comparison of quantized representations and significantly reducing retrieval latency.

For generation, we adopt DDPM and GMM as the parametric and non-parametric baselines, respectively. We train DDPM following the implementation provided by~\citet{fan2024salun}. For GMM, we model the target distribution $p_c(x)$ for each class using 1,000 Gaussian components, with the mean $\mu_i$ and diagonal covariance $\Sigma_i$ of each component estimated using the Expectation-Maximization (EM) algorithm. The resulting model can be expressed as:

\begin{equation}
p_c(x) = \frac{1}{n} \sum_{i=1}^{n} \mathcal{N}(x | \mu_i, \Sigma_i),
\end{equation}
where \( \mathcal{N}(x | \mu_i, \Sigma_i) \) represents the Gaussian distribution with mean \( \mu_i \) and covariance \( \Sigma_i \) for the \( i \)-th component, and \( n \) is the total number of Gaussian components.

\myparagraph{Details for unlearning performance experiments.}
For SPM in classification, we construct the inputted set using the remaining classes, \ie, $\gT \setminus \gU$. Then, we apply the clustering strategy to aggregate the embeddings of samples from each class into an average.

For SPM on generation, if the class is unlearned, the inputted set is composed of images from another remaining class. For example, if the model is supposed to unlearn cat, we can use the images of planes to replace images of cats as the inputted set. Otherwise, the inputted set is sampled from the target class.

\myparagraph{Baselines for unlearning performance.}
 We utilize the code released from~\citet{fan2024salun} and~\citet{wu2025munba} for classification, and from~\citet{huang2024unified} for generation, adhering to their instructions to obtain the baseline results.

\myparagraph{Comparison with EraseDiff} \label{sec:erase_diff}
We utilize the code provided by EraseDiff~\cite{wu2025erasing} and present our results in \tabref{tab:supp_EraseDiff}. However, we were unable to fully replicate the results reported in their original paper. Furthermore, it is worth noting that their study presents results for only a single class in the class-wise unlearning experiment. This limited evaluation makes it challenging to assess the general effectiveness of the approach across a broader range of classes.
\input{tabs/supp_EraseDiff}

\myparagraph{More qualitative results} 
\figref{fig:gene_cifar_more} shows more qualitative results for SPM unlearning, where five CIFAR-10 classes are removed simultaneously. We visualize generated samples for both the unlearned and remaining classes. These results demonstrate that the proposed SPM supports precise, class-level forgetting without harming the generation quality of the remaining categories.
\input{figs/gene_cifar_more_sample}

\section{Proof of Claim~\ref{thm:metric}}
\label{sec:supp_proof}

\claimone*

\begin{proof}

\noindent\textbf{Proof of the first inequality in~\equref{eq:metric-bound}.}
For every $\vx\in\gV$, let $\vy^\star(\vx)$ be the ground-truth label. Then
\bea
\big|\1[\tilde{\vy}(\vx)=\vy^\star(\vx)]-\1[\vy'(\vx)=\vy^\star(\vx)]\big|
\;\le\; \1[\tilde{\vy}(\vx)\neq \vy'(\vx)].
\eea
Averaging over $\vx\in\gV$ yields
\bea
\Delta \Acc(\gV)\;\le\; \text{PG}_{\tt H}(\tilde{F},F').
\eea

\noindent\textbf{Proof of the second inequality in \equref{eq:metric-bound}.}
Fix $\vx\in\gV$ and recall the oracle margin
$\gamma(\vx)=F'(\vx)[y_1]-F'(\vx)[y_2]\ge \gamma_{\min}>0$, where $y_1,y_2$ are the indices of the largest and second largest coordinates of $F'(\vx)$. If $\tilde{\vy}(\vx)=\vy'(\vx)=y_1$, then $\1[\tilde{\vy}(\vx)\neq \vy'(\vx)]=0$. Otherwise, set $\tilde{\vy}=\tilde{\vy}(\vx)\neq y_1$. Since $F'(\vx)[\tilde{\vy}]\le F'(\vx)[y_2]$, we have
\bea
\gamma(\vx)\;\le\; F'(\vx)[y_1]-F'(\vx)[\tilde{\vy}].
\eea
Adding and subtracting $\tilde{F}(\vx)[y_1]$ and $\tilde{F}(\vx)[\tilde{\vy}]$, and using $\tilde{F}(\vx)[y_1]\le \tilde{F}(\vx)[\tilde{\vy}]$, we obtain
\bea
\gamma(\vx)
&\le& \big|F'(\vx)[y_1]-\tilde{F}(\vx)[y_1]\big|
     + \big|F'(\vx)[\tilde{\vy}]-\tilde{F}(\vx)[\tilde{\vy}]\big| \\
&\le& \big\|\tilde{F}(\vx)-F'(\vx)\big\|_1.
\eea
Hence,
\bea
\1[\tilde{\vy}(\vx)\neq \vy'(\vx)]
\;\le\; \frac{\big\|\tilde{F}(\vx)-F'(\vx)\big\|_1}{\gamma_{\min}}.
\eea
By Pinsker’s inequality,
\bea
\big\|\tilde{F}(\vx)-F'(\vx)\big\|_1 \;\le\; \sqrt{2\,\KL\!\big(\tilde{F}(\vx),F'(\vx)\big)}.
\eea
Therefore,
\bea
\1[\tilde{\vy}(\vx)\neq \vy'(\vx)]
\;\le\; \frac{\sqrt{2}}{\gamma_{\min}}\,
\sqrt{\KL\!\big(\tilde{F}(\vx),F'(\vx)\big)}.
\eea
Averaging over $\vx\in\gV$ and applying Jensen’s inequality to $t\mapsto \sqrt{t}$ yields
\bea
\text{PG}_{\tt H}(\tilde{F},F')
\;\le\; \frac{\sqrt{2}}{\gamma_{\min}}\,
\sqrt{\frac{1}{|\gV|}\sum_{\vx\in\gV}\KL\!\big(\tilde{F}(\vx),F'(\vx)\big)}
\;=\; \frac{\sqrt{2}}{\gamma_{\min}}\sqrt{\text{PG}_{\tt S}(\tilde{F},F')}.
\eea
Combining the two parts gives
\bea
\Delta \Acc(\gV)\;\le\; \text{PG}_{\tt H}(\tilde{F},F') \;\le\; \frac{\sqrt{2}}{\gamma_{\min}}\sqrt{\text{PG}_{\tt S}(\tilde{F},F')}.
\eea
\end{proof}

\section{Qualitative Results from GMM}\label{sec:supp_gmm}
\figref{fig:gmm_cifar} presents the qualitative results of GMM on CIFAR-10. Since GMM generates images by superimposing multiple Gaussian components, the generated images exhibit significant noise and blurriness. This highlights the limitations of traditional semi/non-parametric models in handling complex generative tasks, underscoring the necessity of integrating parametric models to effectively address such challenges.
\input{figs/gmm_visual}

\input{tabs/cls_cifar_random}
\section{Random Unlearning Results on Classification} \label{sec:random_unlearn}
In \tabref{tab:cls_cifar_random}, we present the results of random unlearning on CIFAR-10 classification.
Our proposed method achieves low $\text{PG}_{\tt H}$ and $\text{PG}_{\tt S}$ in both the 10\% and 50\% random unlearning scenarios. Moreover, the gaps between our unlearned model and the retrained model in terms of $\Delta$UA, $\Delta$RA, and $\Delta$TA are minimal, demonstrating the effectiveness of our approach.
It is worth noting that evaluating this task solely based on UA is not meaningful, since the unlearned data in random unlearning are essentially unseen samples to the model. Consequently, the model is expected to generalize to such data. Therefore, $\text{PG}_{\tt H}$ and $\text{PG}_{\tt S}$ serve as more appropriate evaluation metrics in this context, as they measure how closely the unlearned model approximates the logits of the retrained model.

%% file: tabs/supp_EraseDiff.tex
\begin{table*}[ht]
\centering
\renewcommand{\arraystretch}{1.2}
\setlength{\tabcolsep}{4pt}
\small
\resizebox{\textwidth}{!}{
\begin{tabular}{c|c|cc|cc|cc|cc|cc|c}
\specialrule{.15em}{.05em}{.05em}
\multirow{2}{*}{\textbf{Backbone}} & \multirow{2}{*}{\textbf{Method}} 
& \multicolumn{2}{c|}{\textbf{Automobile}} 
& \multicolumn{2}{c|}{\textbf{Cat}} 
& \multicolumn{2}{c|}{\textbf{Dog}} 
& \multicolumn{2}{c|}{\textbf{Horse}} 
& \multicolumn{2}{c|}{\textbf{Truck}} 
& \multirow{2}{*}{\textbf{$t_{setup}$  (s) $\downarrow$}} \\
& 
&  \textbf{$\Delta$UA $\downarrow$}  & \textbf{$\Delta$FID$_{\tt R}$ $\downarrow$} 
&  \textbf{$\Delta$UA $\downarrow$}  & \textbf{$\Delta$FID$_{\tt R}$ $\downarrow$} 
&  \textbf{$\Delta$UA $\downarrow$}  & \textbf{$\Delta$FID$_{\tt R}$ $\downarrow$} 
&  \textbf{$\Delta$UA $\downarrow$}  & \textbf{$\Delta$FID$_{\tt R}$ $\downarrow$} 
&  \textbf{$\Delta$UA $\downarrow$}  & \textbf{$\Delta$FID$_{\tt R}$ $\downarrow$} 
&  \\
\hline

\multirow{2}{*}{\shortstack{DDPM}} & \cellcolor{retrain} Retrain & \cellcolor{retrain}  0.00 (100.00) & \cellcolor{retrain} 0.00 (11.21) & \cellcolor{retrain}  0.00 (99.98) & \cellcolor{retrain} 0.00 (10.84) & \cellcolor{retrain} 0.00 (100.00) & \cellcolor{retrain} 0.00 (10.88) & \cellcolor{retrain} 0.00 (99.98) & \cellcolor{retrain} 0.00 (10.00) & \cellcolor{retrain} 0.00 (99.90) & \cellcolor{retrain} 0.00 (10.15) & \cellcolor{retrain} 145601.4 \\ 
& EraseDiff~\cite{wu2025erasing} & 28.06 (71.94) & 2.66 (8.55) & 72.42 (27.56) & 1.96 (8.88) & 51.40 (48.60) & 2.16 (8.72) & 42.24 (57.76) & 1.20 (8.80) & 38.10 (61.80) & 1.80 (8.35) & 1429.6 \\
\hline
\multirow{2}{*}{\shortstack{SPM}} & \cellcolor{retrain} Retrain & \cellcolor{retrain} 0.00 (99.76) & \cellcolor{retrain} 0.00 (7.40) & \cellcolor{retrain} 0.00 (99.26) & \cellcolor{retrain} 0.00 (7.15) & \cellcolor{retrain}  0.00 (99.78)& \cellcolor{retrain} 0.00 (6.52)& \cellcolor{retrain} 0.00 (99.84) & \cellcolor{retrain} 0.00 (6.89) & \cellcolor{retrain}  0.00 (99.64)& \cellcolor{retrain}  0.00 (7.28)& \cellcolor{retrain} 90935.9 \\
& Ours & \textbf{0.10} (99.86) & \textbf{0.38} (7.78) & \textbf{0.04} (99.22) &\textbf{0.50} (7.65) & \textbf{0.04 }(99.82) & \textbf{0.08} (6.60) &\textbf{ 0.08} (99.76) & \textbf{0.29} (7.18) & \textbf{ 0.06 }(99.58) & \textbf{0.40} (7.68) &\textbf{$<$1} \\
\specialrule{.15em}{.05em}{.05em}
\end{tabular}
}
\caption{\textbf{CIFAR-10 class-wise forgetting on generation.} Comparison with EraseDiff.}
\label{tab:supp_EraseDiff}
\end{table*}

%% file: figs/gene_cifar_more_sample.tex
\begin{figure*}[t]
\small
\centering
\renewcommand{\arraystretch}{1.2}
\setlength{\tabcolsep}{2pt}
\begin{tabular}{>{\centering\arraybackslash}m{1.8cm}|*{5}{>{\centering\arraybackslash}m{0.07\textwidth}}|*{5}{>{\centering\arraybackslash}m{0.07\textwidth}}}
\specialrule{.15em}{.05em}{.05em}
  & \multicolumn{5}{c}{\textbf{Class 0 - 4}} & \multicolumn{5}{c}{\textbf{Class 5 - 9}} \\
& planes & cars & birds & cats & deer & dogs & frogs & horses & ships & trucks \\
\midrule
\multirow{6}{*}{\makecell{Pre-trained\\SPM}}
& \includegraphics[width=0.055\textwidth]{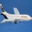} &
  \includegraphics[width=0.055\textwidth]{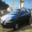} &
  \includegraphics[width=0.055\textwidth]{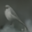} &
    \begin{overpic}[width=0.055\textwidth]{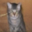} 
\linethickness{0.8pt}
    \put(0,0){\color{blue}\framebox(100,100){}}
\end{overpic} &
  \includegraphics[width=0.055\textwidth]{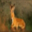} &
  \includegraphics[width=0.055\textwidth]{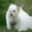} &
    \begin{overpic}[width=0.055\textwidth]{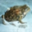} 
\linethickness{0.8pt}
    \put(0,0){\color{blue}\framebox(100,100){}}
\end{overpic} &
  \includegraphics[width=0.055\textwidth]{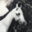} &
  \includegraphics[width=0.055\textwidth]{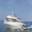} &
  \includegraphics[width=0.055\textwidth]{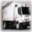} 
  \\ &
  \includegraphics[width=0.055\textwidth]{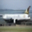} &
  \includegraphics[width=0.055\textwidth]{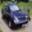} &
  \includegraphics[width=0.055\textwidth]{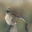} &
    \begin{overpic}[width=0.055\textwidth]{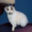} 
\linethickness{0.8pt}
    \put(0,0){\color{blue}\framebox(100,100){}}
\end{overpic} &
  \includegraphics[width=0.055\textwidth]{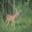} &
  \includegraphics[width=0.055\textwidth]{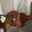} &
    \begin{overpic}[width=0.055\textwidth]{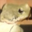} 
\linethickness{0.8pt}
    \put(0,0){\color{blue}\framebox(100,100){}}
\end{overpic} &
  \includegraphics[width=0.055\textwidth]{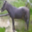} &
  \includegraphics[width=0.055\textwidth]{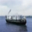} &
  \includegraphics[width=0.055\textwidth]{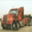} \\&
  \includegraphics[width=0.055\textwidth]{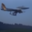} &
  \includegraphics[width=0.055\textwidth]{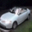} &
  \includegraphics[width=0.055\textwidth]{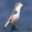} &
    \begin{overpic}[width=0.055\textwidth]{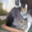} 
\linethickness{0.8pt}
    \put(0,0){\color{blue}\framebox(100,100){}}
\end{overpic} &
  \includegraphics[width=0.055\textwidth]{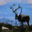} &
  \includegraphics[width=0.055\textwidth]{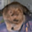} &
    \begin{overpic}[width=0.055\textwidth]{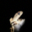} 
\linethickness{0.8pt}
    \put(0,0){\color{blue}\framebox(100,100){}}
\end{overpic} &
  \includegraphics[width=0.055\textwidth]{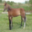} &
  \includegraphics[width=0.055\textwidth]{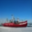} &
  \includegraphics[width=0.055\textwidth]{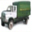} \\
\midrule
\multirow{6}{*}{\makecell{Unlearn \\ 0 - 4}} &
  \includegraphics[width=0.055\textwidth]{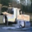} &
  \includegraphics[width=0.055\textwidth]{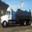} &
  \includegraphics[width=0.055\textwidth]{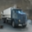} &
\begin{overpic}[width=0.055\textwidth]{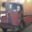}
    \linethickness{0.8pt}
    \put(0,0){\color{red}\framebox(100,100){}}
\end{overpic} &
  \includegraphics[width=0.055\textwidth]{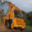} &
  \includegraphics[width=0.055\textwidth]{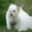} &
  \begin{overpic}[width=0.055\textwidth]{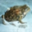} 
\linethickness{0.8pt}
    \put(0,0){\color{blue}\framebox(100,100){}}
\end{overpic} &
  \includegraphics[width=0.055\textwidth]{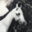} &
  \includegraphics[width=0.055\textwidth]{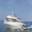} &
  \includegraphics[width=0.055\textwidth]{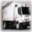} \\
  &
  \includegraphics[width=0.055\textwidth]{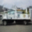} &
  \includegraphics[width=0.055\textwidth]{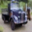} &
  \includegraphics[width=0.055\textwidth]{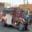} &
\begin{overpic}[width=0.055\textwidth]{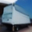}
    \linethickness{0.8pt}
    \put(0,0){\color{red}\framebox(100,100){}}
\end{overpic} &
  \includegraphics[width=0.055\textwidth]{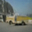} &
  \includegraphics[width=0.055\textwidth]{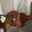} &
  \begin{overpic}[width=0.055\textwidth]{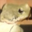} 
\linethickness{0.8pt}
    \put(0,0){\color{blue}\framebox(100,100){}}
\end{overpic} &
  \includegraphics[width=0.055\textwidth]{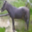} &
  \includegraphics[width=0.055\textwidth]{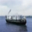} &
  \includegraphics[width=0.055\textwidth]{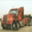}
  \\
  &
  \includegraphics[width=0.055\textwidth]{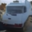} &
  \includegraphics[width=0.055\textwidth]{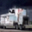} &
  \includegraphics[width=0.055\textwidth]{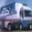} &
\begin{overpic}[width=0.055\textwidth]{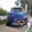}
    \linethickness{0.8pt}
    \put(0,0){\color{red}\framebox(100,100){}}
\end{overpic} &
  \includegraphics[width=0.055\textwidth]{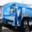} &
  \includegraphics[width=0.055\textwidth]{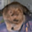} &
  \begin{overpic}[width=0.055\textwidth]{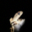} 
\linethickness{0.8pt}
    \put(0,0){\color{blue}\framebox(100,100){}}
\end{overpic} &
  \includegraphics[width=0.055\textwidth]{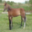} &
  \includegraphics[width=0.055\textwidth]{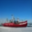} &
  \includegraphics[width=0.055\textwidth]{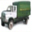}\\
\midrule
\multirow{6}{*}{\makecell{Unlearn \\ 5 - 9}} &
  \includegraphics[width=0.055\textwidth]{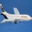} &
  \includegraphics[width=0.055\textwidth]{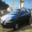} &
  \includegraphics[width=0.055\textwidth]{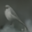} &
      \begin{overpic}[width=0.055\textwidth]{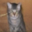} 
\linethickness{0.8pt}
    \put(0,0){\color{blue}\framebox(100,100){}}
\end{overpic} &\includegraphics[width=0.055\textwidth]{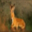} &
  \includegraphics[width=0.055\textwidth]{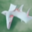} &
  \begin{overpic}[width=0.055\textwidth]{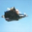}
    \linethickness{0.8pt}
    \put(0,0){\color{red}\framebox(100,100){}}
\end{overpic} &
  \includegraphics[width=0.055\textwidth]{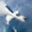} &
  \includegraphics[width=0.055\textwidth]{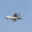} &
  \includegraphics[width=0.055\textwidth]{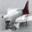} \\
  &
  \includegraphics[width=0.055\textwidth]{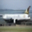} &
  \includegraphics[width=0.055\textwidth]{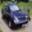} &
  \includegraphics[width=0.055\textwidth]{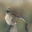} &
      \begin{overpic}[width=0.055\textwidth]{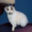} 
\linethickness{0.8pt}
    \put(0,0){\color{blue}\framebox(100,100){}}
\end{overpic} &\includegraphics[width=0.055\textwidth]{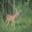} &
  \includegraphics[width=0.055\textwidth]{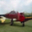} &
  \begin{overpic}[width=0.055\textwidth]{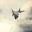}
    \linethickness{0.8pt}
    \put(0,0){\color{red}\framebox(100,100){}}
\end{overpic} &
  \includegraphics[width=0.055\textwidth]{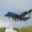} &
  \includegraphics[width=0.055\textwidth]{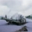} &
  \includegraphics[width=0.055\textwidth]{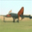}
  \\
  &
  \includegraphics[width=0.055\textwidth]{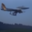} &
  \includegraphics[width=0.055\textwidth]{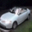} &
  \includegraphics[width=0.055\textwidth]{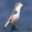} &
      \begin{overpic}[width=0.055\textwidth]{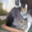} 
\linethickness{0.8pt}
    \put(0,0){\color{blue}\framebox(100,100){}}
\end{overpic} &\includegraphics[width=0.055\textwidth]{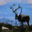} &
  \includegraphics[width=0.055\textwidth]{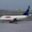} &
  \begin{overpic}[width=0.055\textwidth]{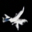}
    \linethickness{0.8pt}
    \put(0,0){\color{red}\framebox(100,100){}}
\end{overpic} &
  \includegraphics[width=0.055\textwidth]{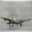} &
  \includegraphics[width=0.055\textwidth]{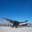} &
  \includegraphics[width=0.055\textwidth]{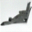}\\
\specialrule{.15em}{.05em}{.05em}
\end{tabular}
\vspace{-0.15cm}
\caption{{\bf Qualitative comparison of unlearned SPM}. %
When a class is unlearned (\eg, cats or frogs framed in \textcolor{red}{red}), the model avoids generating that concept and produces samples resembling remaining classes (\eg, a truck replacing a cat). Additionally, the generations from the remaining classes are left unchanged (framed in \textcolor{blue}{blue}).}
\label{fig:gene_cifar_more}
\vspace{-.1cm}
\end{figure*}

%% file: figs/gmm_visual.tex
\begin{figure*}[h]
\centering
\renewcommand{\arraystretch}{1.2}
\setlength{\tabcolsep}{2pt}
\begin{tabular}{*{10}{>{\centering\arraybackslash}m{0.07\textwidth}}}
\specialrule{.15em}{.05em}{.05em}
  planes & cars & birds & cats & deer & dogs & frogs & horses & ships & trucks \\
\midrule

  \includegraphics[width=0.055\textwidth]{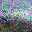} &
  \includegraphics[width=0.055\textwidth]{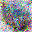} &
  \includegraphics[width=0.055\textwidth]{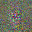} &
  \includegraphics[width=0.055\textwidth]{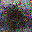} &
  \includegraphics[width=0.055\textwidth]{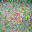} &
  \includegraphics[width=0.055\textwidth]{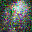} &
  \includegraphics[width=0.055\textwidth]{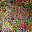} &
  \includegraphics[width=0.055\textwidth]{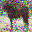} &
  \includegraphics[width=0.055\textwidth]{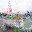} &
  \includegraphics[width=0.055\textwidth]{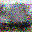} \\
\midrule

  \includegraphics[width=0.055\textwidth]{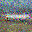} &
  \includegraphics[width=0.055\textwidth]{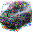} &
  \includegraphics[width=0.055\textwidth]{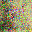} &
  \includegraphics[width=0.055\textwidth]{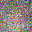} &
  \includegraphics[width=0.055\textwidth]{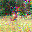} &
  \includegraphics[width=0.055\textwidth]{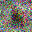} &
  \includegraphics[width=0.055\textwidth]{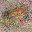} &
  \includegraphics[width=0.055\textwidth]{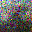} &
  \includegraphics[width=0.055\textwidth]{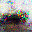} &
  \includegraphics[width=0.055\textwidth]{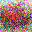} \\
\midrule

  \includegraphics[width=0.055\textwidth]{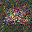} &
  \includegraphics[width=0.055\textwidth]{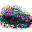} &
  \includegraphics[width=0.055\textwidth]{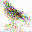} &
  \includegraphics[width=0.055\textwidth]{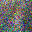} &
  \includegraphics[width=0.055\textwidth]{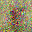} &
  \includegraphics[width=0.055\textwidth]{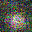} &
  \includegraphics[width=0.055\textwidth]{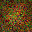} &
  \includegraphics[width=0.055\textwidth]{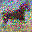} &
  \includegraphics[width=0.055\textwidth]{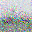} &
  \includegraphics[width=0.055\textwidth]{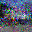} \\
\specialrule{.15em}{.05em}{.05em}
\end{tabular}
\caption{\textbf{Generated images of GMM on CIFAR-10.} The low-quality outputs demonstrate the limitations of traditional non-parametric models in handling complex generative tasks.}
\label{fig:gmm_cifar}
\vspace{-.3cm}
\end{figure*}

%% file: tabs/cls_cifar_random.tex
\begin{table*}[ht]
\centering
\resizebox{\textwidth}{!}{ %
\begin{tabular}{c|c|ccccc|ccccc}
\specialrule{.15em}{.05em}{.05em}
    \multirow{2}{*}{\textbf{Backbone}} & \multirow{2}{*}{\textbf{Method}} & \multicolumn{5}{c|}{\textbf{Random Unlearning (10$\%$)}} & \multicolumn{5}{c}{\textbf{Random Unlearning (50$\%$)}} \\
 &  & \textbf{$\text{PG}_{\tt H}$ $\downarrow$} & \textbf{ $\text{PG}_{\tt S}$$\downarrow$} & \textbf{$\Delta$UA $\downarrow$} & \textbf{$\Delta$RA $\downarrow$} & \textbf{$\Delta$TA $\downarrow$} &   \textbf{$\text{PG}_{\tt H}$ $\downarrow$} & \textbf{ $\text{PG}_{\tt S}$$\downarrow$}  &\textbf{$\Delta$UA $\downarrow$} & \textbf{$\Delta$RA $\downarrow$} & \textbf{$\Delta$TA $\downarrow$}   \\
\hline
\multirow{9}{*}{\shortstack{ResNet18}} & \cellcolor{retrain} Retrain & \cellcolor{retrain} 0.00 & \cellcolor{retrain} 0.00  & \cellcolor{retrain} 0.00 (5.35) &  \cellcolor{retrain} 0.00 (100.00)& \cellcolor{retrain} 0.00 (94.25)  & \cellcolor{retrain} 0.00 & \cellcolor{retrain} 0.00  & \cellcolor{retrain} 0.00 (8.48) &  \cellcolor{retrain} 0.00 (100.00)& \cellcolor{retrain} 0.00 (91.08)  \\
 & GA~\cite{thudi2022unrolling} & \textbf{ 5.23}&0.17  & 4.80 (0.55) & 4.80 (99.52) &  0.48 (94.57)&  8.05&0.37  &7.94 (0.54) & 0.49 (99.51) & 3.45 (94.53) \\
 & FT~\cite{warnecke2021machine} & 5.89 &0.19  & 4.46 (0.89)  & 0.16 (99.84) & 0.32 (93.93) &  8.50&0.36  & 7.52 (0.96)&0.18 (99.82) & 2.67 (93.75) \\
 & IU~\cite{koh2017understanding} &  26.60& 1.69 & 17.15 (22.50) & 22.32 (77.68) & 21.09 (73.16) &  83.63 &8.29  &74.83 (83.32)& 83.29 (16.71)&74.53 (16.55)  \\
 & BE~\cite{chen2023boundary} &  5.30&0.17  &  4.83 (0.52)&0.48 (99.52)  & 0.29 (94.54) & 34.70 &1.22  & 17.33 (25.81)& 25.93 (74.07)&23.37 (67.71) \\
 & BS~\cite{chen2023boundary} &  5.31&0.17  & 4.82 (0.53) &0.49 (99.51)  & 0.28 (94.53)  & 26.55 &0.85  &  8.82 (17.30)& 17.36 (82.64)& 15.39 (75.69)\\
 & $\ell_1$-sparse~\cite{jia2023model} &  6.03&0.19  & 4.29 (1.06) & 0.24 (99.76) &0.50 (93.75)  &  8.53&0.35  & 7.53 (0.95) & \textbf{0.14 }(99.86) &2.67 (93.75)\\
 & SalUn~\cite{fan2024salun} &  5.40&0.15  &  3.89 (1.46)&0.14 (99.86)  &\textbf{0.08 }(94.17)  &  7.94&\textbf{0.33}  & 7.34 (1.14)& 0.52 (99.48)& 2.53 (93.61)\\
 & MUNBa~\cite{wu2025munba} &  5.39&\textbf{0.14}  & 4.39 (0.96)&  \textbf{0.03 }(99.97)&  0.19 (94.44)& 12.94 &1.50   &\textbf{2.62} (5.86) & 3.26 (96.74)& \textbf{2.09 }(88.99) \\
\hline
ResNet18-KNN &  \cellcolor{retrain} Retrain & \cellcolor{retrain} 0.00 & \cellcolor{retrain} 0.00  & \cellcolor{retrain} 0.00 (0.02) &  \cellcolor{retrain} 0.00 (99.98) & \cellcolor{retrain} 0.00 (94.45) & \cellcolor{retrain} 0.00 & \cellcolor{retrain} 0.00  & \cellcolor{retrain}0.00 (0.02) &  \cellcolor{retrain}0.00 (99.98) & \cellcolor{retrain} 0.00 (94.45)\\
\hline
  \multirow{2}{*}{\shortstack{ResNet18-SPM }} 
& \cellcolor{retrain} Retrain  & \cellcolor{retrain} 0.00 & \cellcolor{retrain} 0.00  & \cellcolor{retrain}  0.00 (0.18)&  \cellcolor{retrain} 0.00 (99.43)& \cellcolor{retrain} 0.00 (93.89)  & \cellcolor{retrain} 0.00 & \cellcolor{retrain} 0.00  & \cellcolor{retrain} 0.00 (4.89)&  \cellcolor{retrain} 0.00 (95.96)& \cellcolor{retrain} 0.00 (91.08)\\
& Ours&  5.54&0.19  &\textbf{ 0.16 }(0.02) & 0.55 (99.98) &  0.56 (94.45)&  \textbf{7.83} &0.36  &4.87 (0.02) &  4.02 (99.98) & 3.37 (94.45) \\
\specialrule{.15em}{.05em}{.05em}
\end{tabular}
}
\caption{\textbf{CIFAR-10 random unlearning on classification.} Results are averaged over 10 scenarios. We observe that our method can achieve close performance to the retrained model. The UA, RA, and TA values are reported in parentheses with the corresponding gaps.}
\label{tab:cls_cifar_random}
\end{table*}

%% file: ref.bib
@String(CVPR= {IEEE Conf. Comput. Vis. Pattern Recog.})

@String(ICCV= {Int. Conf. Comput. Vis.})

@String(ECCV= {Eur. Conf. Comput. Vis.})

@String(NIPS= {Adv. Neural Inform. Process. Syst.})

@String(ICLR = {Int. Conf. Learn. Represent.})

@String(AAAI = {AAAI})

@String(CVPR  = {CVPR})

@String(ICCV  = {ICCV})

@String(ECCV  = {ECCV})

@String(NIPS  = {NeurIPS})

@String(ICLR  = {ICLR})

@book{murphy2022probabilistic,
  title={Probabilistic machine learning: an introduction},
  author={Murphy, Kevin P},
  year={2022},
  publisher={MIT press}
}

@article{dempster1977maximum,
  title={Maximum likelihood from incomplete data via the {EM} algorithm},
  author={Dempster, Arthur P and Laird, Nan M and Rubin, Donald B},
  journal={Journal of the royal statistical society},
  year={1977},
}

@inproceedings{heusel2017gans,
  title={{GANs} trained by a two time-scale update rule converge to a local {Nash} equilibrium},
  author={Heusel, Martin and Ramsauer, Hubert and Unterthiner, Thomas and Nessler, Bernhard and Hochreiter, Sepp},
  booktitle=NIPS,
  year={2017}
}

@article{cover1967nearest,
  title={Nearest neighbor pattern classification},
  author={Cover, Thomas and Hart, Peter},
  journal={IEEE TIT},
  year={1967},
}

@book{wasserman2006all,
  title={All of nonparametric statistics},
  author={Wasserman, Larry},
  year={2006},
  publisher={Springer}
}

@book{simonoff2012smoothing,
  title={Smoothing methods in statistics},
  author={Simonoff, Jeffrey S},
  year={2012},
  publisher={Springer Science \& Business Media}
}

@techreport{krizhevsky2009learning,
  title       = {Learning Multiple Layers of Features from Tiny Images},
  author      = {Krizhevsky, Alex},
  year        = {2009},
  institution = {University of Toronto},
  url         = {https://www.cs.toronto.edu/~kriz/learning-features-2009-TR.pdf}
}

@misc{EU2016GDPR,
  title        = {Regulation {(EU)} 2016/679 of the European Parliament and of the Council of 27 {April} 2016 on the protection of natural persons with regard to the processing of personal data and on the free movement of such data (General Data Protection Regulation)},
  author       = {{European Union}},
  year         = {2016},
  note         = {OJ L 119, 4.5.2016, p. 1–88},
  url          = {https://eur-lex.europa.eu/eli/reg/2016/679/oj}
}

@article{battaglia2018relational,
  title={Relational inductive biases, deep learning, and graph networks},
  author={Battaglia, Peter W and Hamrick, Jessica B and Bapst, Victor and Sanchez-Gonzalez, Alvaro and Zambaldi, Vinicius and Malinowski, Mateusz and Tacchetti, Andrea and Raposo, David and Santoro, Adam and Faulkner, Ryan and others},
  journal={arXiv preprint arXiv:1806.01261},
  year={2018}
}

@inproceedings{
kossen2021selfattention,
title={Self-Attention Between Datapoints: Going Beyond Individual Input-Output Pairs in Deep Learning},
author={Jannik Kossen and Neil Band and Clare Lyle and Aidan Gomez and Tom Rainforth and Yarin Gal},
booktitle={Proc. NeurIPS},
year={2021},
}

@inproceedings{blattmann2022retrieval,
  title={Retrieval-augmented diffusion models},
  author={Blattmann, Andreas and Rombach, Robin and Oktay, Kaan and M{\"u}ller, Jonas and Ommer, Bj{\"o}rn},
  booktitle={Proc. NeurIPS},
  volume={35},
  year={2022}
}

@inproceedings{
sheynin2023knndiffusion,
title={K{NN}-Diffusion: Image Generation via Large-Scale Retrieval},
author={Shelly Sheynin and Oron Ashual and Adam Polyak and Uriel Singer and Oran Gafni and Eliya Nachmani and Yaniv Taigman},
booktitle={Proc. ICLR},
year={2023},
}

@inproceedings{lee2019set,
  title={Set transformer: A framework for attention-based permutation-invariant neural networks},
  author={Lee, Juho and Lee, Yoonho and Kim, Jungtaek and Kosiorek, Adam and Choi, Seungjin and Teh, Yee Whye},
  booktitle={Proc. ICML},
  year={2019}
}

@inproceedings{wilson2016deep,
  title={Deep kernel learning},
  author={Wilson, Andrew Gordon and Hu, Zhiting and Salakhutdinov, Ruslan and Xing, Eric P},
  booktitle={Proc. AISTATS},
  year={2016}
}

@inproceedings{qi2018semi,
  title={Semi-parametric image synthesis},
  author={Qi, Xiaojuan and Chen, Qifeng and Jia, Jiaya and Koltun, Vladlen},
  booktitle={Proc. CVPR},
  year={2018}
}

@article{seeger2004gaussian,
  title={Gaussian processes for machine learning},
  author={Seeger, Matthias},
  journal={International journal of neural systems},
  volume={14},
  year={2004},
  publisher={World Scientific}
}

@inproceedings{heng2023selective,
  title={Selective amnesia: A continual learning approach to forgetting in deep generative models},
  author={Heng, Alvin and Soh, Harold},
  booktitle={Proc. NeurIPS},
  volume={36},
  year={2023}
}

@inproceedings{jia2023model,
  title={Model sparsity can simplify machine unlearning},
  author={Jia, Jinghan and Liu, Jiancheng and Ram, Parikshit and Yao, Yuguang and Liu, Gaowen and Liu, Yang and Sharma, Pranay and Liu, Sijia},
  booktitle={Proc. NeurIPS},
  volume={36},
  year={2023}
}

@inproceedings{
    fan2024salun,
    title={Sal{U}n: Empowering Machine Unlearning via Gradient-based Weight Saliency in Both Image Classification and Generation},
    author={Chongyu Fan and Jiancheng Liu and Yihua Zhang and Eric Wong and Dennis Wei and Sijia Liu},
    booktitle={Proc. ICLR},
    year={2024},
}

@article{warnecke2021machine,
  title={Machine unlearning of features and labels},
  author={Warnecke, Alexander and Pirch, Lukas and Wressnegger, Christian and Rieck, Konrad},
  journal={arXiv preprint arXiv:2108.11577},
  year={2021}
}

@inproceedings{thudi2022unrolling,
  title={{Unrolling SGD}: Understanding factors influencing machine unlearning},
  author={Thudi, Anvith and Deza, Gabriel and Chandrasekaran, Varun and Papernot, Nicolas},
  booktitle={Proc. EuroS\&P},
  year={2022},
}

@inproceedings{koh2017understanding,
  title={Understanding black-box predictions via influence functions},
  author={Koh, Pang Wei and Liang, Percy},
  booktitle={Proc. ICML},
  year={2017},
}

@inproceedings{chen2023boundary,
  title={Boundary unlearning: Rapid forgetting of deep networks via shifting the decision boundary},
  author={Chen, Min and Gao, Weizhuo and Liu, Gaoyang and Peng, Kai and Wang, Chen},
  booktitle={Proc. CVPR},
  year={2023}
}

@inproceedings{gandikota2023erasing,
  title={Erasing Concepts from Diffusion Models},
  author={Rohit Gandikota and Joanna Materzy\'nska and Jaden Fiotto-Kaufman and David Bau},
  booktitle={Proc. CVPR},
  year={2023}
}

@inproceedings{golatkar2020eternal,
  title={Eternal sunshine of the spotless net: Selective forgetting in deep networks},
  author={Golatkar, Aditya and Achille, Alessandro and Soatto, Stefano},
  booktitle={Proc. CVPR},
  year={2020}
}

@inproceedings{huang2024unified,
  title={Unified Gradient-Based Machine Unlearning with Remain Geometry Enhancement},
  author={Huang, Zhehao and Cheng, Xinwen and Zheng, JingHao and Wang, Haoran and He, Zhengbao and Li, Tao and Huang, Xiaolin},
  booktitle={Proc. NeurIPS},
  year={2024}
}

@inproceedings{fmn,
    author    = {Zhang, Gong and Wang, Kai and Xu, Xingqian and Wang, Zhangyang and Shi, Humphrey},
    title     = {Forget-Me-Not: Learning to Forget in Text-to-Image Diffusion Models},
    booktitle={Proc. CVPR},
    year      = {2024},
}

@inproceedings{yoshihashi2019classification,
  title={Classification-reconstruction learning for open-set recognition},
  author={Yoshihashi, Ryota and Shao, Wen and Kawakami, Rei and You, Shaodi and Iida, Makoto and Naemura, Takeshi},
  booktitle={Proc. CVPR},
  year={2019}
}

@inproceedings{zaheer2017deep,
  title={Deep sets},
  author={Zaheer, Manzil and Kottur, Satwik and Ravanbakhsh, Siamak and Poczos, Barnabas and Salakhutdinov, Russ R and Smola, Alexander J},
  booktitle={Proc. NeurIPS},
  volume={30},
  year={2017}
}

@inproceedings{vaswani2017attention,
  title={Attention is all you need},
  author={Vaswani, Ashish and Shazeer, Noam and Parmar, Niki and Uszkoreit, Jakob and Jones, Llion and Gomez, Aidan N and Kaiser, {\L}ukasz and Polosukhin, Illia},
  booktitle={Proc. NeurIPS},
  volume={30},
  year={2017}
}

@inproceedings{ddpm,
 author = {Ho, Jonathan and Jain, Ajay and Abbeel, Pieter},
 booktitle = {Proc. NeurIPS},
 title = {Denoising Diffusion Probabilistic Models},
 volume = {33},
 year = {2020}
}

@inproceedings{
ddim,
title={Denoising Diffusion Implicit Models},
author={Jiaming Song and Chenlin Meng and Stefano Ermon},
booktitle={Proc. ICLR},
year={2021},
}

@inproceedings{he2016deep,
  title={Deep residual learning for image recognition},
  author={He, Kaiming and Zhang, Xiangyu and Ren, Shaoqing and Sun, Jian},
  booktitle={Proc. CVPR},
  year={2016}
}

@inproceedings{deng2009imagenet,
  title={Imagenet: A large-scale hierarchical image database},
  author={Deng, Jia and Dong, Wei and Socher, Richard and Li, Li-Jia and Li, Kai and Fei-Fei, Li},
  booktitle={Proc. CVPR},
  year={2009},
}

@article{liu2025rethinking,
  title={Rethinking machine unlearning for large language models},
  author={Liu, Sijia and Yao, Yuanshun and Jia, Jinghan and Casper, Stephen and Baracaldo, Nathalie and Hase, Peter and Yao, Yuguang and Liu, Chris Yuhao and Xu, Xiaojun and Li, Hang and others},
  journal={Nature Machine Intelligence},
  year={2025},
  publisher={Nature Publishing Group UK London}
}

@article{shaik2024exploring,
  title={Exploring the landscape of machine unlearning: A comprehensive survey and taxonomy},
  author={Shaik, Thanveer and Tao, Xiaohui and Xie, Haoran and Li, Lin and Zhu, Xiaofeng and Li, Qing},
  journal={IEEE TNNLS},
  year={2024},
}

@article{wang2024machine,
  title={Machine unlearning: A comprehensive survey},
  author={Wang, Weiqi and Tian, Zhiyi and Zhang, Chenhan and Yu, Shui},
  journal={arXiv preprint arXiv:2405.07406},
  year={2024}
}

@inproceedings{guo2019certified,
  title={Certified data removal from machine learning models},
  author={Guo, Chuan and Goldstein, Tom and Hannun, Awni and Van Der Maaten, Laurens},
  booktitle={Proc. ICML},
  year={2020}
}

@inproceedings{neel2021descent,
  title={Descent-to-delete: Gradient-based methods for machine unlearning},
  author={Neel, Seth and Roth, Aaron and Sharifi-Malvajerdi, Saeed},
  booktitle={Algorithmic Learning Theory},
  year={2021},
  organization={PMLR}
}

@inproceedings{sekhari2021remember,
  title={Remember what you want to forget: Algorithms for machine unlearning},
  author={Sekhari, Ayush and Acharya, Jayadev and Kamath, Gautam and Suresh, Ananda Theertha},
  booktitle={Proc. NeurIPS},
  volume={34},
  year={2021}
}

@inproceedings{ginart2019making,
  title={Making ai forget you: Data deletion in machine learning},
  author={Ginart, Antonio and Guan, Melody and Valiant, Gregory and Zou, James Y},
  booktitle={Proc. NeurIPS},
  volume={32},
  year={2019}
}

@inproceedings{ullah2021machine,
  title={Machine unlearning via algorithmic stability},
  author={Ullah, Enayat and Mai, Tung and Rao, Anup and Rossi, Ryan A and Arora, Raman},
  booktitle={Proc. COLT},
  year={2021},
}

@article{goel2022towards,
  title={Towards adversarial evaluations for inexact machine unlearning},
  author={Goel, Shashwat and Prabhu, Ameya and Sanyal, Amartya and Lim, Ser-Nam and Torr, Philip and Kumaraguru, Ponnurangam},
  journal={arXiv preprint arXiv:2201.06640},
  year={2022}
}

@article{mahadevan2021certifiable,
  title={Certifiable machine unlearning for linear models},
  author={Mahadevan, Ananth and Mathioudakis, Michael},
  journal={arXiv preprint arXiv:2106.15093},
  year={2021}
}

@inproceedings{tarun2023deep,
  title={Deep regression unlearning},
  author={Tarun, Ayush Kumar and Chundawat, Vikram Singh and Mandal, Murari and Kankanhalli, Mohan},
  booktitle={Proc. ICML},
  year={2023}
}

@inproceedings{gandikota2024unified,
  title={Unified concept editing in diffusion models},
  author={Gandikota, Rohit and Orgad, Hadas and Belinkov, Yonatan and Materzy{\'n}ska, Joanna and Bau, David},
  booktitle={Proc. WACV},
  year={2024}
}

@article{brundage2018malicious,
  title={The malicious use of artificial intelligence: Forecasting, prevention, and mitigation},
  author={Brundage, Miles and Avin, Shahar and Clark, Jack and Toner, Helen and Eckersley, Peter and Garfinkel, Ben and Dafoe, Allan and Scharre, Paul and Zeitzoff, Thomas and Filar, Bobby and others},
  journal={arXiv preprint arXiv:1802.07228},
  year={2018}
}

@article{marchal2024generative,
  title={Generative {AI} misuse: A taxonomy of tactics and insights from real-world data},
  author={Marchal, Nahema and Xu, Rachel and Elasmar, Rasmi and Gabriel, Iason and Goldberg, Beth and Isaac, William},
  journal={arXiv preprint arXiv:2406.13843},
  year={2024}
}

@techreport{ISRSAA2025,
title = {International {AI} Safety Report},
author = {Bengio, Yoshua and Mindermann, S{\"o}ren and Privitera, Daniel and Besiroglu, Tamay and
Bommasani, Rishi and others},
year = {2025},
number = {DSIT 2025/001},
URL = {https://www.gov.uk/government/publications/international-ai-safety-report-2025}
}

@article{gyevnar2025ai,
  title={AI safety for everyone},
  author={Gyevnar, Balint and Kasirzadeh, Atoosa},
  journal={Nature Machine Intelligence},
  year={2025},
  publisher={Nature Publishing Group UK London}
}

@article{barez2025open,
  title={Open problems in machine unlearning for ai safety},
  author={Barez, Fazl and Fu, Tingchen and Prabhu, Ameya and Casper, Stephen and Sanyal, Amartya and Bibi, Adel and O'Gara, Aidan and Kirk, Robert and Bucknall, Ben and Fist, Tim and others},
  journal={arXiv preprint arXiv:2501.04952},
  year={2025}
}

@inproceedings{cao2015towards,
  title={Towards making systems forget with machine unlearning},
  author={Cao, Yinzhi and Yang, Junfeng},
  booktitle={IEEE SSP},
  year={2015}
}

@inproceedings{wu2022puma,
  title={{PUMA}: Performance unchanged model augmentation for training data removal},
  author={Wu, Ga and Hashemi, Masoud and Srinivasa, Christopher},
  booktitle={Proc. AAAI},
  year={2022}
}

@article{nguyen2020variational,
  title={Variational {Bayesian} unlearning},
  author={Nguyen, Quoc Phong and Low, Bryan Kian Hsiang and Jaillet, Patrick},
  journal={Proc. NeurIPS},
  year={2020}
}

@inproceedings{brack2023sega,
  title={{SEGA}: Instructing text-to-image models using semantic guidance},
  author={Brack, Manuel and Friedrich, Felix and Hintersdorf, Dominik and Struppek, Lukas and Schramowski, Patrick and Kersting, Kristian},
  booktitle={Proc. NeurIPS},
  year={2023}
}

@inproceedings{schramowski2023safe,
  title={Safe latent diffusion: Mitigating inappropriate degeneration in diffusion models},
  author={Schramowski, Patrick and Brack, Manuel and Deiseroth, Bj{\"o}rn and Kersting, Kristian},
  booktitle={Proc. CVPR},
  year={2023}
}

@inproceedings{kim2023towards,
  title={Towards safe self-distillation of internet-scale text-to-image diffusion models},
  author={Kim, Sanghyun and Jung, Seohyeon and Kim, Balhae and Choi, Moonseok and Shin, Jinwoo and Lee, Juho},
  booktitle={ICML Workshop on Challenges in Deployable Generative AI},
  year={2023}
}

@inproceedings{kumari2023ablating,
  title={Ablating concepts in text-to-image diffusion models},
  author={Kumari, Nupur and Zhang, Bingliang and Wang, Sheng-Yu and Shechtman, Eli and Zhang, Richard and Zhu, Jun-Yan},
  booktitle={Proc. ICCV},
  year={2023}
}

@inproceedings{wu2025munba,
  title={{MUNBa}: Machine unlearning via {Nash} bargaining},
  author={Wu, Jing and Harandi, Mehrtash},
  booktitle={Proc. ICCV},
  year={2025}
}

@inproceedings{wu2025erasing,
  title={Erasing undesirable influence in diffusion models},
  author={Wu, Jing and Le, Trung and Hayat, Munawar and Harandi, Mehrtash},
  booktitle={Proc. CVPR},
  year={2025}
}

@article{douze2025faiss,
  title={The {FAISS} library},
  author={Douze, Matthijs and Guzhva, Alexandr and Deng, Chengqi and Johnson, Jeff and Szilvasy, Gergely and Mazar{\'e}, Pierre-Emmanuel and Lomeli, Maria and Hosseini, Lucas and J{\'e}gou, Herv{\'e}},
  journal={IEEE TBD},
  year={2025},
}

@inproceedings{unet,
  title={{U-Net}: Convolutional networks for biomedical image segmentation},
  author={Ronneberger, Olaf and Fischer, Philipp and Brox, Thomas},
  booktitle={Proc. MICCAI},
  year={2015}
}

@inproceedings{bahdanau2014neural,
  title={Neural machine translation by jointly learning to align and translate},
  author={Bahdanau, Dzmitry and Cho, Kyunghyun and Bengio, Yoshua},
    booktitle={Proc. ICLR},
  year={2015}
}

@inproceedings{bourtoule2021machine,
  title={Machine unlearning},
  author={Bourtoule, Lucas and Chandrasekaran, Varun and Choquette-Choo, Christopher A and Jia, Hengrui and Travers, Adelin and Zhang, Baiwu and Lie, David and Papernot, Nicolas},
  booktitle={IEEE SSP},
  year={2021}
}

@inproceedings{zheng2024imma,
  title={Imma: Immunizing text-to-image models against malicious adaptation},
  author={Zheng, Amber Yijia and Yeh, Raymond A},
  booktitle={Proc. ECCV},
  year={2024}
}

@inproceedings{zheng2024learning,
  title={Learning to obstruct few-shot image classification over restricted classes},
  author={Zheng, Amber Yijia and Yang, Chiao-An and Yeh, Raymond A},
  booktitle={Proc. ECCV},
  year={2024}
}

@inproceedings{zheng2025multi,
  title={Multi-concept model immunization through differentiable model merging},
  author={Zheng, Amber Yijia and Yeh, Raymond A},
  booktitle={Proc. AAAI},
  volume={39},
  year={2025}
}

@inproceedings{
    zheng2025model,
    title={Model Immunization from a Condition Number Perspective},
    author={Amber Yijia Zheng and Site Bai and Brian Bullins and Raymond A. Yeh},
    booktitle={Proc. ICML},
    year={2025},
}
